\documentclass{article} % For LaTeX2e
\usepackage{iclr2027_conference,times}

\usepackage{amsmath,amsfonts,bm}

\def\eqref#1{equation~\ref{#1}}
\def\1{\bm{1}}

\DeclareMathAlphabet{\mathsfit}{\encodingdefault}{\sfdefault}{m}{sl}
\SetMathAlphabet{\mathsfit}{bold}{\encodingdefault}{\sfdefault}{bx}{n}

\usepackage{hyperref}       % hyperlinks
\usepackage{url}            % simple URL typesetting
\usepackage{booktabs}       % professional-quality tables
\usepackage{amsfonts}       % blackboard math symbols
\usepackage{nicefrac}       % compact symbols for 1/2, etc.
\usepackage{microtype}      % microtypography
\usepackage{xcolor}         % colors
\usepackage{algorithm}

\usepackage{graphicx}
\usepackage{subcaption}

\usepackage{mathtools}
\usepackage{paralist}
\usepackage{multicol, multirow}
\usepackage{caption}
\usepackage{makecell}
\usepackage{enumitem}
\usepackage{algorithmic}
\usepackage{float}
\usepackage{wrapfig}

\usepackage{amsmath}
\usepackage{amssymb}
\usepackage{amsthm}

\theoremstyle{plain}
\newtheorem{theorem}{Theorem}[section]

\newtheorem{proposition}[theorem]{Proposition}
 
\theoremstyle{definition}

\theoremstyle{remark}

\title{Beyond Simulation: Retain-and-Repair Neural Operators for Real-World Adaptation}
\iclrfinalcopy
\author{Woojin Cho \\
TelePIX \\
\texttt{woojin@telepix.net} \\
\And
Junghwan Park \\
TelePIX \\
\texttt{junghwan@telepix.net} \\
}

\begin{document}

\maketitle
\begin{abstract}
Neural operators increasingly benefit from pretraining on numerical simulations, yet adapting them for real-world prediction remains challenging. We introduce the Retain-and-Repair Neural Operator ($\text{R}^{2}\text{NO}$), a framework for adapting simulation-pretrained operators to real-world data while retaining useful pretrained structure. The pretrained operator is first finetuned on real data and then frozen to provide a source prediction, and a shared repair module learns a sequence of refinements from the same observations. Using orthogonal Fourier projections, a spectral ensemble fits a small ridge regression within each cell of the Fourier domain and combines the refinements by weights fitted on a held-out split of the real data. The cells are defined jointly by radial ranges, angular sectors, and measured channels, allowing refinement depth to vary with frequency magnitude, with orientation, and across channels. Including the source prediction as a candidate makes retention available in every cell, and independently trained repair modules enter the same combination as additional candidates. On all RealPDEBench systems and six backbones, $\text{R}^{2}\text{NO}$ consistently outperforms full finetuning and iterative refinement. The framework treats adaptation depth as a cell-specific choice learned from real data.
\end{abstract}

\section{Introduction}
\label{sec:introduction}

Neural operators learn mappings between function spaces and provide fast surrogates for the evolution of physical systems~\citep{li2020fourier,lu2021learning,raonic2023convolutional}. Much of their development has relied on numerical simulations for both training and evaluation, while recent pretraining approaches have learned reusable representations across diverse simulated systems~\citep{hao2024dpot,herde2024poseidon}. Even when simulations and real measurements nominally follow the same governing equations, modeling approximations, experimental conditions, and measurement processes can create a substantial sim-to-real gap. However, adapting neural operators to limited real measurements remains largely underexplored. This motivates a central question: how can we adapt simulation-trained operators to real measurements without discarding what they already get right?

Looped models provide a natural structure for addressing this question. Recent approaches in language and diffusion modeling repeatedly apply a weight-shared block to produce intermediate representations without increasing the number of parameters~\citep{giannou2023looped,saunshi2025reasoning, movahedi2026fixed}. In neural operators, IRNO~\citep{liu2026iterative} similarly applies iterative corrections to mitigate spectral bias. We use this iterative structure for a different purpose: determining which components of a simulation-pretrained prediction should be retained and which should be repaired using real observations.

We analyze this behavior in detail in Section~\ref{sec:motivation}, showing that different spectral regions along a fixed refinement trajectory favor different depths. Consequently, a single iterate trades accuracy in one region against accuracy in another. The Fourier domain is particularly suitable for resolving this tradeoff because disjoint spectral regions induce mutually orthogonal projections, allowing the squared prediction error to be decomposed exactly across regions.
Motivated by this observation, we introduce the Retain-and-Repair Neural Operator ($\text{R}^{2}\text{NO}$), a model-agnostic framework for adapting simulation-pretrained neural operators to real measurements. 
$\text{R}^{2}\text{NO}$ adapts in two phases: it first finetunes the pretrained operator on real observations, then freezes the result and applies a shared repair module trained on the same observations, producing a trajectory of candidate predictions.
A spectral ensemble then combines these candidates separately within Fourier cells defined by radial frequency, orientation, and channel. Because the original prediction remains an explicit candidate, each cell can retain transferable simulation-derived components while using the refinement trajectory to repair components that do not transfer. The cell-wise weights are fitted through small ridge-regression problems on real data, and independently trained repair trajectories can be incorporated as additional candidates. Since $\text{R}^{2}\text{NO}$ operates on predicted fields, it can be applied across neural operator backbones without modifying their internal architectures.

We make the following contributions.
\begin{itemize}
\item \textbf{Model-agnostic retain-and-repair adaptation.} We introduce a framework that adapts a simulation-pretrained neural operator using a weight-shared repair module trained on limited real observations, while preserving the source prediction as an explicit candidate.

\item \textbf{Region-wise spectral ensemble and analysis.} We fit region-specific combinations of the source prediction and its repairs over a radial--angular spectral partition. We characterize the advantage over using a common refinement depth and distinguish convergence of the repair trajectory from the accuracy of its finite iterates.

\item \textbf{Evaluation under real-world finetuning.} We evaluate $\text{R}^{2}\text{NO}$ on all five RealPDEBench systems~\citep{hu2026realpdebench} and six pretrained backbones, comparing it with full finetuning and iterative refinement under the same sim-to-real adaptation setting.
\end{itemize}

\section{Motivation: Frequency-Dependent Refinement}
\label{sec:motivation}

A shared refinement module produces several predictions for the same input, but how many times to apply it is normally decided once for the whole output field. Methods that adapt the depth do so per token~\citep{elbayad2019depth,bae2026mixture}, not across the frequency content of a field. We ask whether every spectral component of a real-world prediction favors the same depth.

\paragraph{Iterative residual refinement.}
The iterative refinement neural operator (IRNO) of~\citet{liu2026iterative} augments a pretrained neural operator with a weight-shared residual loop, so one trained model yields a prediction at every depth. The operator $\Psi(\cdot;\pi)$ stays frozen and maps a real input history $\mathcal{X}$ to an initial prediction $h_{0}=\Psi(\mathcal{X};\pi)$. A refinement module $\Phi(\cdot;\theta)$ receives the input history together with the current prediction and returns a residual of the same shape as the output field, which is added with a step size $\alpha>0$.
\begin{equation}
h_{\ell+1}=h_{\ell}+\alpha\,\Phi(\mathcal{X},h_{\ell};\theta),
\qquad \ell=0,\ldots,L-1.
\label{eq:irno-trajectory}
\end{equation}
Because $\Phi$ predicts a correction rather than the field itself, the pretrained prediction is the starting point of the loop and every iterate $h_{\ell}$ is a complete prediction, while running the loop longer adds computation but no parameters. IRNO trains $\Phi$ with an objective that supervises every iterate against the target, and at inference it reports the field at one depth chosen for the whole output.

\begin{figure}[t]
\centering
\subfloat[Real observation]{\includegraphics[width=0.325\columnwidth]{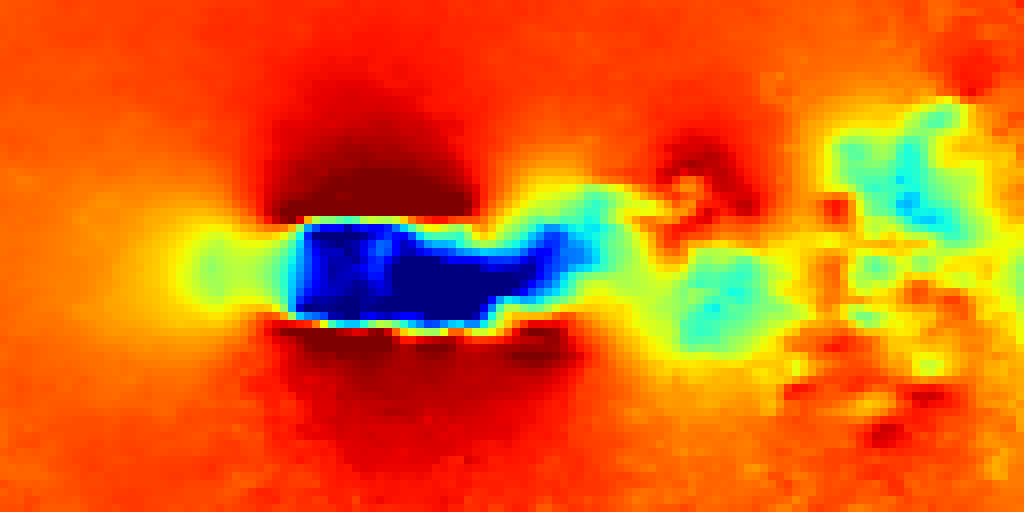}}\hfill
\subfloat[Radial Fourier regions\label{fig:dist_fourier-bands}]{\includegraphics[width=0.325\columnwidth]{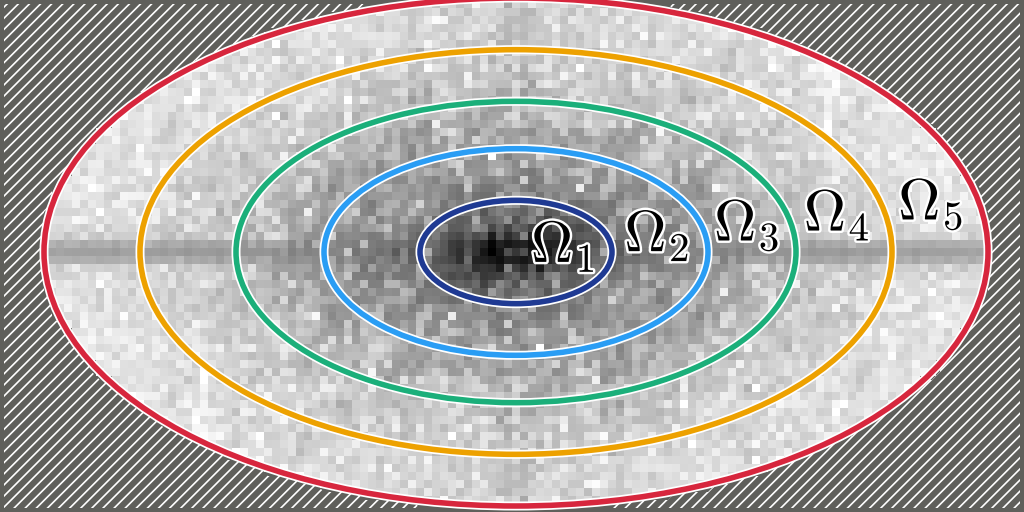}}\hfill
\subfloat[Refinement error by region]{\includegraphics[width=0.325\columnwidth]{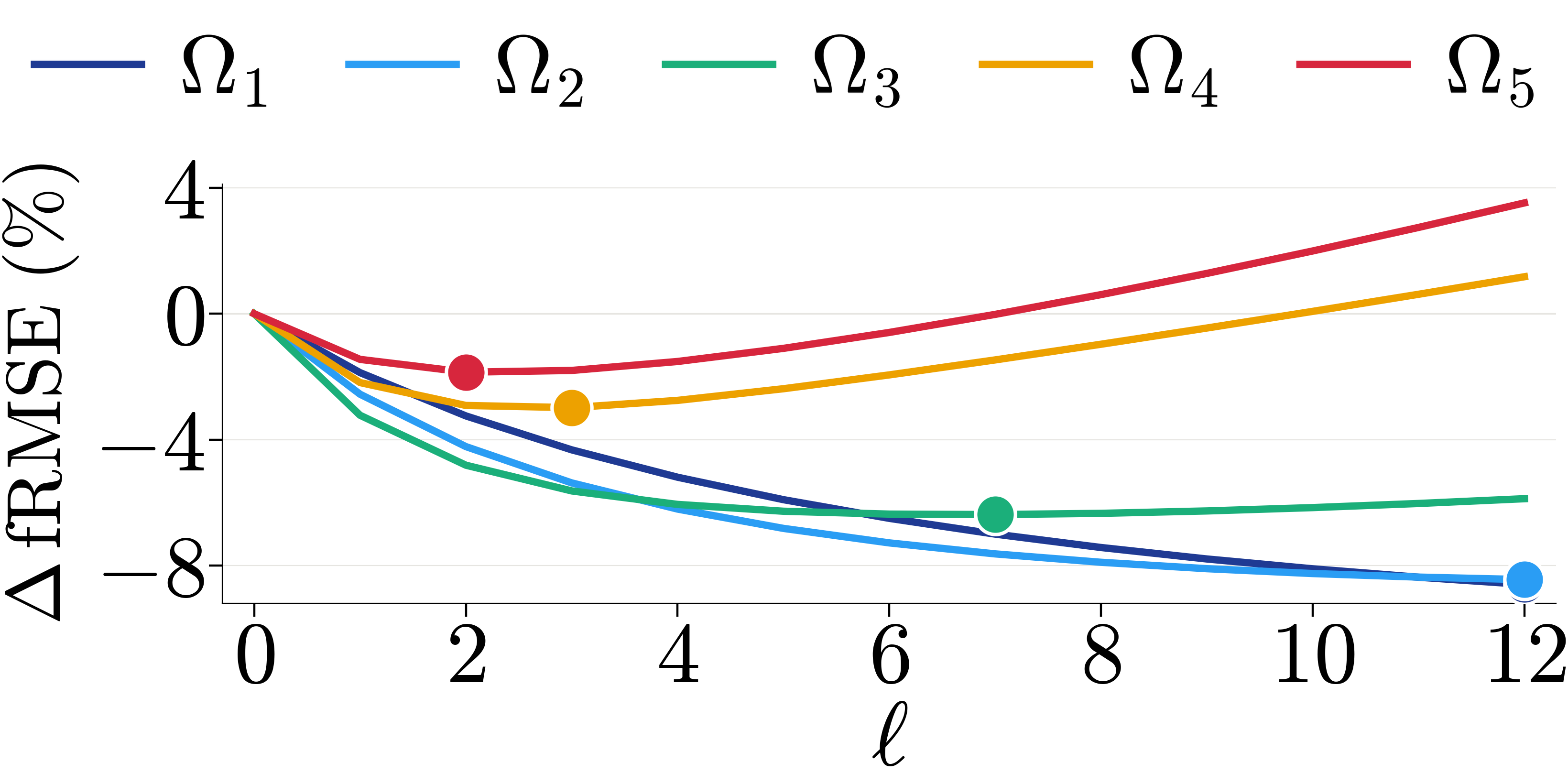}}
\caption{
\textbf{Different Fourier regions favor different refinement depths.}
(a) A real Cylinder observation in physical space.
(b) Five radial Fourier regions, ordered from low to high frequency.
(c) Region-wise fRMSE change (\%) relative to $\ell=0$ along an IRNO trajectory trained on FNO; markers indicate minima over the evaluated depths, and lower is better. Setup in Appendix~\ref{app:diagnostic-setup}.
}
\label{fig:dist_fourier}
\end{figure}

\paragraph{Refinement trained on real data.}
For a RealPDEBench system, we take the finetuned backbone as $\Psi$, train $\Phi$ on the real training trajectories with the IRNO objective, and then hold both $\pi$ and $\theta$ fixed. 
Rather than scoring the whole field at one depth, we look at how the error of each iterate is distributed over the Fourier domain.

\paragraph{Error by spectral region.}
\label{sec:spectral-setup}
Let $\mathcal{F}$ be the unitary Fourier transform over the spatial axes of the output, and let $\Omega_{1},\ldots,\Omega_{N_{\Omega}}$ be disjoint, conjugate-symmetric sets of wavenumbers. Region $b$ acts on a field by keeping only its Fourier coefficients in $\Omega_{b}$.
\begin{equation}
P_{b}v=\mathcal{F}^{-1}\!\left[\mathbf{1}_{\Omega_{b}}\odot\mathcal{F}v\right],
\qquad b=1,\ldots,N_{\Omega}.
\label{eq:projector}
\end{equation}
The projections in Eq.~\ref{eq:projector} are mutually orthogonal, since $\mathcal{F}$ is unitary and the sets $\Omega_{b}$ are disjoint, and conjugate symmetry keeps $P_{b}v$ real for real $v$. We run the loop of Eq.~\ref{eq:irno-trajectory} for $L=12$ steps on the real test split $\mathcal{D}^{\mathrm{test}}$ and compute the fRMSE~\citep{takamoto2022pdebench, hu2026realpdebench} of every iterate $h_{0},\ldots,h_{L}$ in the $N_{\Omega}=5$ radial regions of Figure~\ref{fig:dist_fourier}(b); Appendix~\ref{app:diagnostic-setup} details the setup. Figure~\ref{fig:dist_fourier}(c) plots the change relative to $\ell=0$; negative values mean improvement, and the markers indicate each region's best depth. On the real Cylinder data, this best depth decreases from the lowest band $\Omega_{1}$ to the highest $\Omega_{5}$: no single $\ell$ serves the whole spectrum. This is an observation about one trained trajectory rather than a rule tied to frequency; what it argues for is learning each region's depth from data instead of prescribing it.

\paragraph{The cost of a common depth.}
The disagreement can be quantified in the squared field error that our spectral ensemble later minimizes. Let $\widehat{\mathbb{E}}_{\mathrm{test}}$ denote the empirical average over the input--target pairs $(\mathcal{X},y)$ of $\mathcal{D}^{\mathrm{test}}$, let all norms be Euclidean norms of the vectorized field, and let $\mathcal{S}_{0}=\{0,\ldots,L\}$ be the set of candidate depths, which includes the source prediction at $\ell=0$. For every candidate we record its squared error restricted to region $b$.
\begin{equation}
\widehat{E}_{b}(\ell)=\widehat{\mathbb{E}}_{\mathrm{test}}\,\|P_{b}(h_{\ell}-y)\|_{2}^{2},
\qquad \ell\in\mathcal{S}_{0}.
\label{eq:diagnostic-region-error}
\end{equation}
When the regions form an exhaustive partition of the represented wavenumbers, every field is the sum of its regional components, and orthogonality splits its squared error exactly across regions.
\begin{equation}
\sum_{b=1}^{N_{\Omega}}P_{b}v=v,
\qquad
\|v-y\|_{2}^{2}=\sum_{b=1}^{N_{\Omega}}\|P_{b}(v-y)\|_{2}^{2}.
\label{eq:parseval}
\end{equation}
Averaged over $\mathcal{D}^{\mathrm{test}}$, the error of candidate $h_{\ell}$ is therefore $\sum_{b}\widehat{E}_{b}(\ell)$, and the price of one depth on every region is the gap between the best common depth and independent region-wise selection.
\begin{equation}
\min_{\ell\in\mathcal{S}_{0}}\sum_{b=1}^{N_{\Omega}}\widehat{E}_{b}(\ell)
-\sum_{b=1}^{N_{\Omega}}\min_{\ell\in\mathcal{S}_{0}}\widehat{E}_{b}(\ell)
=\min_{\ell\in\mathcal{S}_{0}}\sum_{b=1}^{N_{\Omega}}
\Bigl[\widehat{E}_{b}(\ell)-\min_{\ell'\in\mathcal{S}_{0}}\widehat{E}_{b}(\ell')\Bigr]\geq0.
\label{eq:depth-composition-gap}
\end{equation}
Each bracket is nonnegative, so the gap is strictly positive exactly when no single depth minimizes every regional error, and the subtracted term is attained by the composite field $\sum_{b}P_{b}h_{\ell_{b}}$, with $\ell_{b}$ a minimizer of $\widehat{E}_{b}$, since Eq.~\ref{eq:parseval} applies to that field as well. The five radial bands serve only to expose the frequency dependence. The spectral ensemble of Section~\ref{sec:spectral-ensemble} applies the same principle to an exhaustive partition of the represented Fourier grid, replaces the selection of one candidate per band by a fitted combination of all of them, fits that combination on a split of the real data disjoint from the one used to train the repair module, and keeps the source prediction among the candidates so that any cell may retain it exactly.

\section{Method: Retain-and-Repair Neural Operators}
\label{sec:method}

Motivated by Section~\ref{sec:motivation}, $\text{R}^{2}\text{NO}$ does not commit to one refinement depth. It keeps the whole trajectory of the repair loop, the source prediction included, and lets real data decide, cell by cell in the Fourier domain, how much of each to use.

\subsection{Overview}
\label{sec:setup}

Figure~\ref{fig:main} shows the two stages. Section~\ref{sec:repair-loop} describes the retain-and-repair loop, which runs the refinement module of Eq.~\ref{eq:irno-trajectory} from the frozen source prediction and keeps its iterates as candidates, and Section~\ref{sec:spectral-ensemble} describes the spectral ensemble, which fits, for each Fourier cell, how the output draws on those candidates and on the source prediction. The two stages use disjoint real data. The repair modules are trained on the training split $\mathcal{D}^{\mathrm{train}}$ of the benchmark, and the ensemble is fitted on its validation split $\mathcal{D}^{\mathrm{fit}}$, on whose targets the modules are never trained. Fitting on the training split instead would be biased, because the repairs look better on their own training data than on new data and an ensemble fitted there would over-trust them relative to the source prediction. The validation split also selects the checkpoint of each repair module among those saved during training, and Appendix~\ref{app:analysis} accounts for this choice. The test split $\mathcal{D}^{\mathrm{test}}$ is used only for evaluation.
The pretrained neural operator $\Psi(\cdot;\pi)$ is first finetuned on $\mathcal{D}^{\mathrm{train}}$ and then frozen, and it maps a real input history $\mathcal{X}\in\mathcal{A}$ to the source prediction $h_{0}=\Psi(\mathcal{X};\pi)\in\mathcal{B}$, where $\mathcal{A}$ and $\mathcal{B}$ are the input and output function spaces in Figure~\ref{fig:main}. This $h_{0}$ coincides with the prediction of the full finetuning baseline, although any other checkpoint could serve as the source. Appendix~\ref{app:pretrained-source} evaluates the simulation-pretrained checkpoint as the source. Since the benchmark's systems report different variables at different resolutions, the modality and resolution mappings it prescribes are applied to the source prediction first, so that $h_{0}$ and the real target $y$ share the same variables, forecast frames, and spatial grid. All residuals and Fourier cells below are defined on this common grid.

\begin{figure}[t]
\centering
\includegraphics[width=\linewidth]{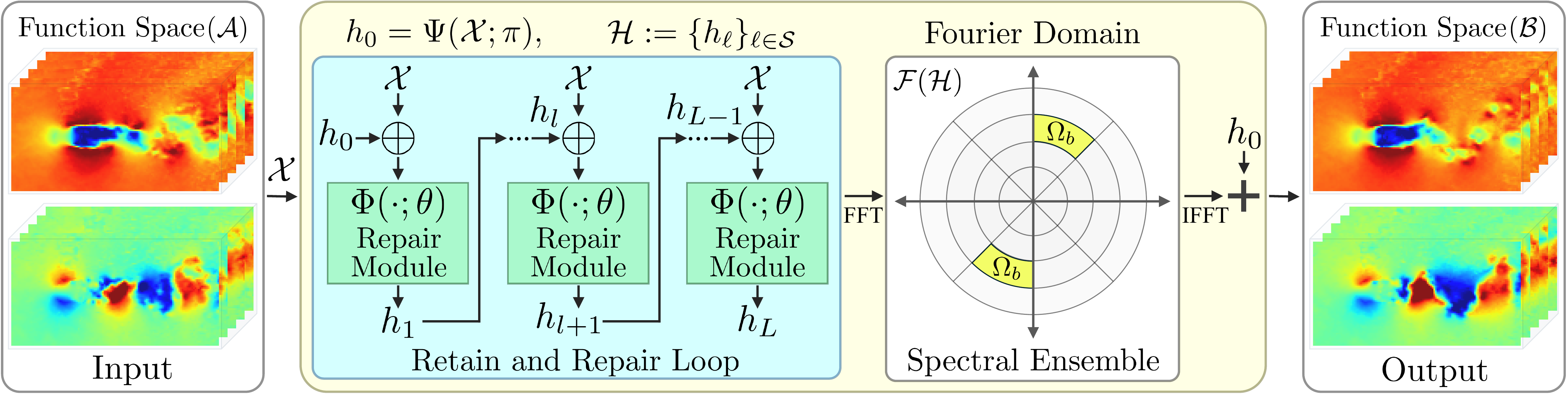}
\caption{\textbf{Overall pipeline of $\text{R}^{2}\text{NO}$.} The retain-and-repair loop applies the shared repair module $L$ times to the frozen source prediction $h_{0}$ and keeps every iterate. The spectral ensemble then fits, for each Fourier cell, how much of the source prediction and of each repair to use.}
\label{fig:main}
\end{figure}

\subsection{Retain-and-Repair Loop on Real Data}
\label{sec:repair-loop}
The repair module $\Phi(\cdot;\theta)$ is applied $L$ times, starting from the source prediction $h_{0}$ with a step size $\alpha>0$, and every iterate is kept as a candidate. When $M$ modules are trained independently from different random seeds, member $m$ produces its own trajectory with parameters $\theta^{(m)}$, and the superscript is dropped for $M=1$.
\begin{equation}
h^{(m)}_{\ell+1}=h^{(m)}_{\ell}+\alpha\,\Phi\!\left(\mathcal{X},h^{(m)}_{\ell};\theta^{(m)}\right),
\qquad h^{(m)}_{0}=h_0,
\qquad \ell=0,\ldots,L-1.
\label{eq:repair}
\end{equation}
The source prediction enters Eq.~\ref{eq:repair} only as its starting point. Retention is enforced later, at the ensemble stage, through the residual form of Eq.~\ref{eq:spectral-output}, so the learned updates remain free to change any part of the field. In Figure~\ref{fig:main}, $\oplus$ denotes the concatenation of the module's inputs.

\paragraph{Training objective.}
Each module is trained on $\mathcal{D}^{\mathrm{train}}$ with the objective and optimization procedure of the iterative refinement baseline~\citep{liu2026iterative}, so that any difference in accuracy is attributable to the ensemble rather than to a modified training signal. The objective supervises every iterate of the loop with three terms: a spatial residual, a spectral residual, and a fixed-point penalty. All three are computed on the channels the benchmark measures (two of the three predicted channels for the fluid systems, one of sixteen for combustion), and in this paragraph $\|v\|^{2}$ denotes the mean of $v^{2}$ over those channels, the forecast frames, and the spatial grid, a mean rather than the Euclidean sum of Section~\ref{sec:motivation}, so that the terms have comparable scale.

The spatial residual of iterate $\ell$ is $\|h^{(m)}_{\ell}-y\|^{2}$. The spectral residual compares the magnitudes of the Fourier coefficients of the iterate and of the target, so it constrains the amplitude spectrum and leaves the phase to the spatial residual. It is averaged over the half-plane $\mathcal{K}_{+}$ of wavenumbers returned by the real-input Fourier transform, the other half being redundant for real fields, and over channels and frames as above, under a weighting $\mu_{\ell}$ that increases with the radius $\rho(k)\in[0,\sqrt{2}]$ of the wavenumber normalized per axis by the Nyquist wavenumber:
\begin{equation}
\mathcal{L}_{\mathrm{spe}}^{\ell}=\frac{1}{\lvert\mathcal{K}_{+}\rvert}\sum_{k\in\mathcal{K}_{+}}\mu_{\ell}(k)
\Bigl(\bigl|[\mathcal{F}h^{(m)}_{\ell}](k)\bigr|-\bigl|[\mathcal{F}y](k)\bigr|\Bigr)^{2},
\qquad
\mu_{\ell}(k)\propto1+\rho(k)^{\eta_{\ell}},
\label{eq:training-terms}
\end{equation}
with $\mu_{\ell}$ in Eq.~\ref{eq:training-terms} normalized to unit mean over $\mathcal{K}_{+}$. The exponent $\eta_{\ell}=1+\frac{\ell-1}{L-1}$ rises from one at the first iterate to two at the last, so later iterates are pushed harder on high wavenumbers. The additive unit keeps the lowest wavenumbers from being discarded, and the weight at the corner of the grid is $1+2^{\eta_{\ell}/2}$ times that at the zero wavenumber, between $2.41$ and $3.00$. The unit-mean normalization keeps the overall scale of the term independent of $\eta_{\ell}$. The fixed-point penalty $\|\Phi(\mathcal{X},y;\theta^{(m)})\|^{2}$ evaluates the module at the real target and penalizes the correction it would still apply there, so that the target is encouraged to be a fixed point of the repair map~\citep{hsieh2019learning}.

Writing $\widehat{\mathbb{E}}_{\mathrm{train}}$ for the empirical average over $\mathcal{D}^{\mathrm{train}}$, the objective averages the spatial and spectral residuals over the $L$ iterates and adds the penalty, with $\beta_{\mathrm{spe}},\beta_{\mathrm{fp}}>0$ weighting the spectral residual and the fixed-point penalty against the spatial residual:
\begin{equation}
\mathcal{L}\!\left(\theta^{(m)}\right)=
\widehat{\mathbb{E}}_{\mathrm{train}}\!\left[\frac{1}{L}\sum_{\ell=1}^{L}
\Bigl(\bigl\|h^{(m)}_{\ell}-y\bigr\|^{2}+\beta_{\mathrm{spe}}\,\mathcal{L}_{\mathrm{spe}}^{\ell}\Bigr)
+\beta_{\mathrm{fp}}\,\bigl\|\Phi\!\left(\mathcal{X},y;\theta^{(m)}\right)\bigr\|^{2}\right].
\label{eq:training}
\end{equation}
Only $\theta^{(m)}$ is updated by minimizing Eq.~\ref{eq:training}; the pretrained operator remains fixed.

\subsection{Spectral Ensemble Across Refinement Depths}
\label{sec:spectral-ensemble}

\paragraph{Fourier cells.}
The regions of Section~\ref{sec:motivation} are refined into cells that separate frequency magnitude, orientation, and channel. For a spatial wavenumber $k$, let $\rho(k)$ be the normalized radius of Section~\ref{sec:repair-loop}, and let $\varphi(k)\in[0,1)$ be the orientation of $k$ modulo a half-turn, expressed as a fraction of the half-turn. The spectrum is divided into $N_{\rho}$ radial bands and $N_{\varphi}$ angular sectors, and the band and the sector containing $k$ are
\begin{equation}
b_{\rho}(k)=\min\bigl(\lfloor\rho(k)\,N_{\rho}\rfloor,\,N_{\rho}-1\bigr),\qquad
b_{\varphi}(k)=\min\bigl(\lfloor\varphi(k)\,N_{\varphi}\rfloor,\,N_{\varphi}-1\bigr).
\label{eq:cell-index}
\end{equation}
The bands have equal width in $\rho$ and the sectors equal width in $\varphi$. The clipping in Eq.~\ref{eq:cell-index} assigns the coefficients beyond the inscribed ellipse of the grid, where $\rho(k)\geq1$, to the outermost band. A cell consists of the coefficients of one measured channel that share a band and a sector, so there are $N_{\Omega}=N_{\rho}N_{\varphi}N_{\mathrm{ch}}$ cells for every system, in place of the five bands of Section~\ref{sec:motivation}. The projection $P_{b}$ applies the mask of Eq.~\ref{eq:projector} to the channel of cell $b$ and returns zero on the other channels. Since $k$ and $-k$ share an orientation modulo a half-turn, a coefficient and its conjugate always fall in the same cell, which is why the two highlighted sectors in Figure~\ref{fig:main} share one weight and every $P_{b}$ maps real fields to real fields. The cells are disjoint, cover every coefficient of the measured channels, and are shared across forecast frames.

\paragraph{Candidate columns.}
Let $\mathcal{S}\subseteq\mathcal{S}_{0}\setminus\{0\}$ be the retained depths and $\mathcal{H}:=\{h_{\ell}\}_{\ell\in\mathcal{S}}$ the retained iterates of a module, as drawn in Figure~\ref{fig:main}. The columns of the ensemble are the residuals of every retained iterate relative to the source prediction, together with one base column,
\begin{equation}
c_{1},\dots,c_{J}\in\bigl\{h^{(m)}_{\ell}-h_{0}:\ m\leq M,\ \ell\in\mathcal{S}\bigr\}\cup\{-h_{0}\},\qquad J=M\lvert\mathcal{S}\rvert+1.
\label{eq:columns}
\end{equation}
Without the base column, a cell could only form affine combinations of $h_{0}$ and the iterates, and the base column lets it also attenuate the source prediction where its content does not transfer. The repair budget is $M$ times the largest retained depth in evaluations of $\Phi$, since every shallower depth is produced on the way at no extra cost.

\paragraph{Predictor and closed-form fit.}
The predictor is a residual correction of the source prediction, in which every cell adds its own combination of the columns.
\begin{equation}
\mathcal{R}(\mathcal{X})=h_{0}+\sum_{b=1}^{N_{\Omega}}\sum_{j=1}^{J}w_{b,j}\,P_{b}c_{j}.
\label{eq:spectral-output}
\end{equation}
All-zero weights reproduce $h_{0}$ exactly, and since every cell lies in a measured channel, the unmeasured channels of the output are always those of $h_{0}$. The weights $w_{b}\in\mathbb{R}^{J}$ of each cell are chosen to minimize the squared error on the fitting split, compared on the measured channels, with $\widehat{\mathbb{E}}_{\mathrm{fit}}$ the empirical average over $\mathcal{D}^{\mathrm{fit}}$. By Eq.~\ref{eq:parseval} and the mutual orthogonality of the $P_{b}$, this error decomposes into one term per cell, and each term involves only the weights of that cell.
\begin{equation}
\widehat{\mathbb{E}}_{\mathrm{fit}}\bigl\|y-\mathcal{R}(\mathcal{X})\bigr\|_{2}^{2}
=\sum_{b=1}^{N_{\Omega}}\widehat{\mathbb{E}}_{\mathrm{fit}}
\Bigl\|P_{b}(y-h_{0})-\sum_{j=1}^{J}w_{b,j}\,P_{b}c_{j}\Bigr\|_{2}^{2}.
\label{eq:combination-objective}
\end{equation}
The fit therefore decouples into $N_{\Omega}$ independent least-squares problems of dimension $J$, and each of them is solved in closed form. Successive iterates are nearly collinear, so a plain least-squares solution would be unstable, and a ridge penalty is added. Let $\langle\cdot,\cdot\rangle$ denote the Euclidean inner product of vectorized fields. The Gram matrix of cell $b$ collects the inner products of the projected columns with one another, and the right-hand side collects their inner products with the projected target residual.
\begin{equation}
[\mathbf{G}_{b}]_{jj'}=\widehat{\mathbb{E}}_{\mathrm{fit}}\bigl\langle P_{b}c_{j},P_{b}c_{j'}\bigr\rangle,\qquad
[\mathbf{r}_{b}]_{j}=\widehat{\mathbb{E}}_{\mathrm{fit}}\bigl\langle P_{b}(y-h_{0}),P_{b}c_{j}\bigr\rangle.
\label{eq:normal-equations}
\end{equation}
The ridge term is scaled by the mean diagonal entry of $\mathbf{G}_{b}$ so that the regularization is comparable across cells, and the weights of cell $b$ follow by solving a $J\times J$ linear system.
\begin{equation}
w_{b}=\Bigl(\mathbf{G}_{b}+\lambda\,\frac{\operatorname{tr}\mathbf{G}_{b}}{J}\,\mathbf{I}\Bigr)^{-1}\mathbf{r}_{b}.
\label{eq:ridge-solution}
\end{equation}
Here $\mathbf{I}$ is the identity matrix and $\lambda>0$ is the ridge parameter, chosen once by solving on one half of the fitting split and scoring on the other over a fixed logarithmic grid. Eq.~\ref{eq:ridge-solution} is then re-solved on the whole split, and the weights are held fixed afterwards. No gradient descent is involved. Because $\mathcal{F}$ is unitary, $\|P_{b}v\|_{2}=\|\mathbf{1}_{\Omega_{b}}\odot\mathcal{F}v\|_{2}$, so in practice the Gram matrices are accumulated from the masked Fourier coefficients of one forward transform per field, and the output is assembled with a single inverse transform, as drawn in Figure~\ref{fig:main}. All fields enter the fit in physical, denormalized units.

\paragraph{Properties of the fitted ensemble.}
The weights $w_{b}$ are unconstrained, so a cell can reproduce any single column with a one-hot $w_{b}$ or interpolate between columns, and neighboring cells may make different choices without any prescribed relation between position in the spectrum and depth. Setting $N_{\Omega}=1$ recovers a single global combination. Since Eq.~\ref{eq:combination-objective} is minimized cell by cell over all of these choices, at $\lambda=0$ the fitting error of $\mathcal{R}$ is no larger than that of the source prediction, of any single candidate $h^{(m)}_{\ell}$, or of any cell-wise selection among them, and when the source checkpoint is the finetuned baseline, no larger than that of the baseline either. This is the ensemble counterpart of the gap in Eq.~\ref{eq:depth-composition-gap}, stated for combinations rather than selections, and Appendix~\ref{app:analysis} gives the proof. The guarantee concerns the fitting split. The ridge term trades a controlled part of it for stability, and because $\mathcal{D}^{\mathrm{fit}}$ is disjoint from the split on which the candidates were trained, Appendix~\ref{app:analysis} also bounds how a cell-wise selection generalizes beyond it.

\paragraph{Multiple repair modules.}
For $M>1$ the modules share the frozen source checkpoint and differ only in their random seed, as in deep ensembles~\citep{lakshminarayanan2017simple}, and their iterates enter Eq.~\ref{eq:columns} as additional columns of one joint fit. Since uniform averaging is one admissible choice of $w_{b}$ in Eq.~\ref{eq:combination-objective}, the joint fit is never worse on the fitting split than averaging $M$ separately fitted ensembles, and it lets a cell take a different depth from each member. The two enlargements differ in cost: adding depths to $\mathcal{S}$ reuses iterates already produced and only enlarges the $J\times J$ systems, whereas each additional member costs a training run and a rollout at inference.

Finally, convergence of the repair loop and accuracy of its iterates are different properties. Under a local contraction assumption, the loop converges to a fixed point whose distance to the target is controlled by the correction the module would still apply at the target, and even then the best depth can differ across regions. Appendix~\ref{app:analysis} makes both statements precise. The ensemble therefore combines $h_0$ and its repairs without requiring $\mathcal{R}$ to be a fixed point of any repair map.

\begin{table}[t]
\centering
\caption{Real-world adaptation on RealPDEBench. All adapted methods start from the same finetuned checkpoint within each backbone, IRNO is reported at its final depth and at the depth $\ell^{\star}$ selected on the fitting split, and $M$ counts the repair modules of $\text{R}^{2}\text{NO}$. Lower is better, and the best value of each column within a backbone is set in bold, with ties marked jointly.}
\label{tab:main}
\small
\setlength{\tabcolsep}{2.5pt}
\renewcommand{\arraystretch}{1.04}
\resizebox{\columnwidth}{!}{%
\begin{tabular}{@{}ll*{10}{c}@{}}
\toprule
& & \multicolumn{2}{c}{Cylinder}
    & \multicolumn{2}{c}{\shortstack{Controlled\\Cylinder}}
    & \multicolumn{2}{c}{FSI}
    & \multicolumn{2}{c}{Foil}
    & \multicolumn{2}{c}{Combustion} \\
\cmidrule(lr){3-4}\cmidrule(lr){5-6}\cmidrule(lr){7-8}
\cmidrule(lr){9-10}\cmidrule(lr){11-12}
Backbone & Adaptation
    & \shortstack{RMSE\\$(\times10^{-2})$} & \shortstack{fRMSE\\$(\times10^{-3})$}
    & \shortstack{RMSE\\$(\times10^{-2})$} & \shortstack{fRMSE\\$(\times10^{-3})$}
    & \shortstack{RMSE\\$(\times10^{-2})$} & \shortstack{fRMSE\\$(\times10^{-3})$}
    & \shortstack{RMSE\\$(\times10^{-2})$} & \shortstack{fRMSE\\$(\times10^{-3})$}
    & \shortstack{RMSE\\$(\times10^{-2})$} & \shortstack{fRMSE\\$(\times10^{-3})$} \\
\midrule
\multirow{5}{*}{U-Net}
    & Full finetuning & 6.315 & 9.731 & 0.786 & 0.924 & 0.841 & 0.714 & 0.936 & 0.825 & 2.134 & 2.487 \\
    & IRNO ($\ell{=}12$) & 5.982 & 9.198 & 0.788 & 0.929 & 0.876 & 0.739 & 0.935 & 0.820 & 2.159 & 2.526 \\
    & IRNO ($\ell^{\star}$) & 5.970 & 9.172 & 0.778 & 0.912 & 0.838 & 0.700 & 0.922 & 0.819 & 2.133 & 2.489 \\
    & $\text{R}^{2}\text{NO}$ ($M{=}1$) & 3.510 & 5.181 & \textbf{0.776} & 0.907 & \textbf{0.836} & \textbf{0.697} & 0.911 & 0.797 & 2.121 & 2.466 \\
    & $\text{R}^{2}\text{NO}$ ($M{=}3$) & \textbf{2.775} & \textbf{3.912} & \textbf{0.776} & \textbf{0.905} & \textbf{0.836} & 0.698 & \textbf{0.905} & \textbf{0.784} & \textbf{2.119} & \textbf{2.462} \\
\midrule
\multirow{5}{*}{DeepONet}
    & Full finetuning & 6.608 & 10.698 & 2.913 & 5.186 & 3.315 & 4.762 & 2.256 & 2.893 & 2.270 & 2.746 \\
    & IRNO ($\ell{=}12$) & 4.061 & 6.179 & 0.956 & 1.234 & 1.114 & 1.147 & 1.477 & 1.900 & 2.188 & 2.593 \\
    & IRNO ($\ell^{\star}$) & 4.061 & 6.179 & 0.956 & 1.234 & 1.114 & 1.147 & 1.477 & 1.900 & 2.188 & 2.593 \\
    & $\text{R}^{2}\text{NO}$ ($M{=}1$) & 3.121 & 3.984 & 0.854 & 1.029 & 0.971 & 0.856 & 1.316 & 1.611 & 2.159 & 2.536 \\
    & $\text{R}^{2}\text{NO}$ ($M{=}3$) & \textbf{2.703} & \textbf{3.315} & \textbf{0.831} & \textbf{0.981} & \textbf{0.917} & \textbf{0.775} & \textbf{1.131} & \textbf{1.321} & \textbf{2.123} & \textbf{2.457} \\
\midrule
\multirow{5}{*}{FNO}
    & Full finetuning & 5.449 & 8.661 & 0.943 & 1.077 & 1.266 & 1.162 & 1.199 & 1.241 & 2.246 & 2.726 \\
    & IRNO ($\ell{=}12$) & 4.920 & 7.984 & 0.818 & 0.956 & 1.038 & 1.010 & 1.157 & 1.200 & 2.261 & 2.716 \\
    & IRNO ($\ell^{\star}$) & 4.913 & 7.972 & 0.818 & 0.956 & 1.037 & 1.010 & 1.130 & 1.122 & 2.239 & 2.707 \\
    & $\text{R}^{2}\text{NO}$ ($M{=}1$) & 3.912 & 5.665 & 0.813 & 0.947 & 1.019 & 0.970 & 1.081 & 1.055 & 2.207 & 2.671 \\
    & $\text{R}^{2}\text{NO}$ ($M{=}3$) & \textbf{3.432} & \textbf{4.919} & \textbf{0.804} & \textbf{0.934} & \textbf{0.985} & \textbf{0.927} & \textbf{1.061} & \textbf{1.019} & \textbf{2.201} & \textbf{2.659} \\
\midrule
\multirow{5}{*}{CNO}
    & Full finetuning & 3.524 & 5.709 & 0.798 & 0.903 & 0.960 & 0.779 & 1.145 & 1.343 & 2.370 & 2.954 \\
    & IRNO ($\ell{=}12$) & 3.223 & 4.687 & 0.785 & 0.902 & 0.858 & 0.710 & 1.062 & 1.119 & 2.220 & 2.632 \\
    & IRNO ($\ell^{\star}$) & 3.223 & 4.683 & 0.782 & 0.898 & 0.858 & 0.709 & 1.040 & 1.099 & 2.219 & 2.637 \\
    & $\text{R}^{2}\text{NO}$ ($M{=}1$) & 2.838 & 3.846 & \textbf{0.781} & \textbf{0.896} & 0.856 & 0.702 & 0.996 & 1.019 & 2.199 & 2.606 \\
    & $\text{R}^{2}\text{NO}$ ($M{=}3$) & \textbf{2.525} & \textbf{3.220} & \textbf{0.781} & \textbf{0.896} & \textbf{0.848} & \textbf{0.693} & \textbf{0.979} & \textbf{0.953} & \textbf{2.164} & \textbf{2.534} \\
\midrule
\multirow{5}{*}{Transolver}
    & Full finetuning & 8.641 & 12.887 & 1.665 & 2.156 & 2.260 & 2.507 & 1.965 & 2.243 & 3.744 & 4.943 \\
    & IRNO ($\ell{=}12$) & 4.912 & 7.750 & 0.883 & 1.021 & 1.024 & 0.822 & 1.243 & 1.370 & 2.442 & 2.858 \\
    & IRNO ($\ell^{\star}$) & 4.912 & 7.750 & 0.883 & 1.021 & 1.024 & 0.822 & 1.243 & 1.370 & 2.396 & 2.770 \\
    & $\text{R}^{2}\text{NO}$ ($M{=}1$) & 3.305 & 4.439 & 0.841 & 0.976 & 0.964 & 0.752 & 1.127 & 1.178 & 2.209 & 2.604 \\
    & $\text{R}^{2}\text{NO}$ ($M{=}3$) & \textbf{2.421} & \textbf{3.011} & \textbf{0.819} & \textbf{0.934} & \textbf{0.948} & \textbf{0.731} & \textbf{1.073} & \textbf{1.095} & \textbf{2.179} & \textbf{2.542} \\
\midrule
\multirow{5}{*}{DPOT-S}
    & Full finetuning & 4.404 & 6.726 & 0.847 & 0.987 & 0.990 & 0.737 & 1.093 & 1.178 & 2.106 & 2.408 \\
    & IRNO ($\ell{=}12$) & 4.151 & 6.399 & 0.799 & 0.943 & 0.869 & 0.733 & 1.057 & 1.155 & 2.108 & 2.350 \\
    & IRNO ($\ell^{\star}$) & 4.151 & 6.399 & 0.799 & 0.943 & 0.869 & 0.733 & 1.057 & 1.155 & 2.082 & 2.343 \\
    & $\text{R}^{2}\text{NO}$ ($M{=}1$) & 3.224 & 4.990 & 0.795 & 0.931 & 0.866 & 0.729 & 1.037 & 1.140 & 2.064 & 2.327 \\
    & $\text{R}^{2}\text{NO}$ ($M{=}3$) & \textbf{2.617} & \textbf{3.665} & \textbf{0.789} & \textbf{0.914} & \textbf{0.860} & \textbf{0.726} & \textbf{1.016} & \textbf{1.102} & \textbf{2.058} & \textbf{2.318} \\
\bottomrule
\end{tabular}}
\end{table}

\section{Experiments}
\label{sec:experiments}
In this section, we compare $\text{R}^{2}\text{NO}$ with the baselines on RealPDEBench and then examine the contribution of each of its components through ablations.
\subsection{Real-World Adaptation on RealPDEBench}
\label{sec:real-world-adaptation}

We evaluate whether the spectral ensemble improves real-world prediction across physical systems and pretrained backbones, following the real-world finetuning task of RealPDEBench~\citep{hu2026realpdebench}, in which simulation-pretrained operators are adapted on real training trajectories and evaluated on held-out real observations. The five systems are Cylinder, Controlled Cylinder, fluid--structure interaction (FSI), Foil, and Combustion, and we retain the data partitions, preprocessing, observed variables, conditioning information, input and output windows, and evaluation procedure of the benchmark. Its training split serves as $\mathcal{D}^{\mathrm{train}}$ for the repair modules, its validation split serves as the fitting split $\mathcal{D}^{\mathrm{fit}}$ on which the spectral ensemble and its ridge parameter are solved, and its test split is used only to compute the reported errors. Appendix~\ref{app:setup} records the protocol in full.

\paragraph{Backbones and adaptation methods.} We evaluate six backbones, namely U-Net~\citep{ronneberger2015u}, DeepONet~\citep{lu2021learning}, CNO~\citep{raonic2023convolutional}, FNO~\citep{li2020fourier}, Transolver~\citep{wu2024transolver}, and DPOT-S~\citep{hao2024dpot}, and Table~\ref{tab:main} compares five adaptation settings for each of them. Full finetuning takes a backbone pretrained on simulation data and updates all its parameters on the real training split, and the resulting checkpoint serves as the frozen source operator of every other setting, so its prediction is exactly the source prediction $h_{0}$. IRNO~\citep{liu2026iterative} trains a shared refinement module on top of this checkpoint with the objective of Eq.~\ref{eq:training}, and we report it at its final iterate $h_{12}$ and at the iterate $h_{\ell^{\star}}$ whose depth $\ell^{\star}\in\{0,\dots,12\}$ minimizes the RMSE on the fitting split, which realizes the best common depth of Eq.~\ref{eq:depth-composition-gap} on the same data as the spectral ensemble. $\text{R}^{2}\text{NO}$ with $M=1$ combines all twelve iterates of the same module with the source prediction through the spectral ensemble of Section~\ref{sec:spectral-ensemble}, and $\text{R}^{2}\text{NO}$ with $M=3$ adds two repair modules trained with different seeds and fits a single ensemble over the columns of all three.

\paragraph{Comparison protocol.} IRNO and $\text{R}^{2}\text{NO}$ with $M=1$ share the frozen checkpoint, the trained repair module, and the twelve evaluations of $\Phi$, so they differ only in how the trajectory is read out, and their comparison isolates the contribution of the spectral ensemble. The setting with $M=3$ uses three training runs and three times this repair budget, and Section~\ref{sec:ablation} separates its gain from the additional budget. Appendix~\ref{app:setup} lists the remaining implementation choices. 

\paragraph{Metrics.} We report the root mean squared error (RMSE) and the frequency-domain root mean squared error (fRMSE) of~\citet{takamoto2022pdebench}, both on the measured channels of the test split. The fRMSE is sensitive to spectral content and therefore complements the RMSE for a method that composes its output in the Fourier domain. Appendix~\ref{app:additional-tables} reports the relative $L_{2}$ error.

\paragraph{Main results.} $\text{R}^{2}\text{NO}$ with $M=3$ attains the lowest RMSE in all thirty pairs of Table~\ref{tab:main}, up to ties at the reported precision, and the lowest fRMSE in twenty-nine, the exception being U-Net on FSI, where the single-module setting is lower by under $0.2$ percent. At a matched budget, $\text{R}^{2}\text{NO}$ with $M=1$ improves on IRNO at either depth in every pair on both metrics while reading the same trajectory. Selecting the depth of IRNO on the fitting split changes little, since the final depth is selected in fourteen of the thirty pairs and the summed RMSE and fRMSE fall by only $0.4$ percent each, and summed over the thirty pairs $\text{R}^{2}\text{NO}$ with $M=1$ still lowers the RMSE of the tuned baseline by $14.5$ percent and its fRMSE by $20.6$ percent. The gain therefore comes from combining depths cell by cell rather than from choosing one depth well, and the larger reduction in fRMSE agrees with the diagnostic of Section~\ref{sec:motivation}, because a single depth over-refines high-frequency components that the ensemble can instead draw from earlier iterates or the source prediction. Adding two repair modules lowers the summed errors by a further $8.1$ and $11.7$ percent.

\paragraph{Retention of the source prediction.} At its final depth, IRNO fails to improve on full finetuning in five of the thirty pairs under RMSE and in three under fRMSE, namely for U-Net on FSI, Controlled Cylinder, and Combustion and, under RMSE, for DPOT-S and FNO on Combustion, where the source prediction is already accurate and a common refinement depth degrades components that required no repair. Selecting the depth removes these failures only by halting the repair of the entire field after two to four steps, which retains less of the improvement that $\text{R}^{2}\text{NO}$ with $M=1$ obtains in the same pairs and still leaves U-Net on Combustion worse than full finetuning under fRMSE. $\text{R}^{2}\text{NO}$ instead retains the source prediction only in cells where it is accurate and repairs the others, and it improves on full finetuning in every pair with either number of repair modules. This is the behavior that Proposition~\ref{prop:composition} guarantees on the fitting split, since the all-zero weighting that reproduces the source prediction is always feasible, and it carries over to the test split in every case we evaluate.
\begin{table}[t]
\centering
\caption{Ablation of the combination rule at a matched repair budget. All rows use the same three repair modules rolled out to $L=12$ and differ only in how their iterates are combined. Lower is better, and the best value of each column within a backbone is set in bold.}
\label{tab:ablation}
\small
\setlength{\tabcolsep}{2.5pt}
\renewcommand{\arraystretch}{1.04}
\resizebox{\columnwidth}{!}{%
\begin{tabular}{@{}ll*{10}{c}@{}}
\toprule
& & \multicolumn{2}{c}{Cylinder}
    & \multicolumn{2}{c}{\shortstack{Controlled\\Cylinder}}
    & \multicolumn{2}{c}{FSI}
    & \multicolumn{2}{c}{Foil}
    & \multicolumn{2}{c}{Combustion} \\
\cmidrule(lr){3-4}\cmidrule(lr){5-6}\cmidrule(lr){7-8}
\cmidrule(lr){9-10}\cmidrule(lr){11-12}
Backbone & Combination
    & \shortstack{RMSE\\$(\times10^{-2})$} & \shortstack{fRMSE\\$(\times10^{-3})$}
    & \shortstack{RMSE\\$(\times10^{-2})$} & \shortstack{fRMSE\\$(\times10^{-3})$}
    & \shortstack{RMSE\\$(\times10^{-2})$} & \shortstack{fRMSE\\$(\times10^{-3})$}
    & \shortstack{RMSE\\$(\times10^{-2})$} & \shortstack{fRMSE\\$(\times10^{-3})$}
    & \shortstack{RMSE\\$(\times10^{-2})$} & \shortstack{fRMSE\\$(\times10^{-3})$} \\
\midrule
\multirow{4}{*}{U-Net}
    & Mean of final iterates & 5.841 & 9.018 & 0.783 & 0.920 & 0.867 & 0.737 & 0.927 & 0.820 & 2.141 & 2.497 \\
    & Mean of all iterates & 5.973 & 9.211 & 0.778 & 0.912 & 0.841 & 0.702 & 0.922 & 0.820 & 2.134 & 2.490 \\
    & Global combination & 5.213 & 8.110 & 0.778 & 0.912 & 0.837 & 0.699 & 0.921 & 0.818 & 2.131 & 2.497 \\
    & $\text{R}^{2}\text{NO}$ & \textbf{2.775} & \textbf{3.912} & \textbf{0.776} & \textbf{0.905} & \textbf{0.836} & \textbf{0.698} & \textbf{0.905} & \textbf{0.784} & \textbf{2.119} & \textbf{2.462} \\
\midrule
\multirow{4}{*}{FNO}
    & Mean of final iterates & 4.868 & 7.843 & 0.811 & 0.952 & 1.017 & 0.991 & 1.135 & 1.152 & 2.248 & 2.700 \\
    & Mean of all iterates & 5.012 & 8.067 & 0.828 & 0.981 & 1.040 & 1.029 & 1.127 & 1.122 & 2.240 & 2.700 \\
    & Global combination & 4.463 & 6.925 & 0.808 & 0.946 & 1.010 & 0.971 & 1.120 & 1.100 & 2.235 & 2.715 \\
    & $\text{R}^{2}\text{NO}$ & \textbf{3.432} & \textbf{4.919} & \textbf{0.804} & \textbf{0.934} & \textbf{0.985} & \textbf{0.927} & \textbf{1.061} & \textbf{1.019} & \textbf{2.201} & \textbf{2.659} \\
\midrule
\multirow{4}{*}{DPOT-S}
    & Mean of final iterates & 4.080 & 6.501 & 0.793 & 0.937 & 0.866 & 0.729 & 1.051 & 1.147 & 2.100 & 2.331 \\
    & Mean of all iterates & 4.168 & 6.569 & 0.801 & 0.947 & 0.886 & 0.731 & 1.055 & 1.151 & 2.082 & 2.329 \\
    & Global combination & 3.829 & 5.990 & 0.793 & 0.931 & 0.864 & 0.729 & 1.048 & 1.139 & 2.071 & 2.333 \\
    & $\text{R}^{2}\text{NO}$ & \textbf{2.617} & \textbf{3.665} & \textbf{0.789} & \textbf{0.914} & \textbf{0.860} & \textbf{0.726} & \textbf{1.016} & \textbf{1.102} & \textbf{2.058} & \textbf{2.318} \\
\bottomrule
\end{tabular}}
\end{table}

\subsection{Ablations}
\label{sec:ablation}

\paragraph{Combination rule at a matched budget.} Table~\ref{tab:ablation} asks whether the gain of $\text{R}^{2}\text{NO}$ with $M=3$ comes from the additional repair modules or from how their iterates are combined, so every row uses the same three modules and thirty-six evaluations of $\Phi$ on U-Net, FNO, and DPOT-S and differs only in the combination rule. The mean of final iterates averages the three final iterates, the mean of all iterates averages all thirty-six, and the global combination fits Eq.~\ref{eq:ridge-solution} on the single cell $N_{\Omega}=1$ with the same ridge selection. Relative to full finetuning and summed over the fifteen pairs, these rules reduce the RMSE by $6.4$, $5.3$, and $10.9$ percent and the fRMSE by $5.5$, $4.4$, and $11.5$ percent, whereas the full partition reaches $26.4$ and $32.8$ percent, and weighting every depth alike is worse than using the final depth alone. $\text{R}^{2}\text{NO}$ attains the lowest error in every column, so the dominance that Proposition~\ref{prop:composition} establishes on the fitting split, where both uniform averages are feasible, also holds on the test split. Removing any single axis of the partition raises the summed RMSE from $0.232$ to between $0.241$ and $0.248$, as reported in Appendix~\ref{app:additional-ablations}.

\begin{wrapfigure}{r}{0.415\textwidth}
\vspace{-10pt}
\centering
\includegraphics[width=\linewidth]{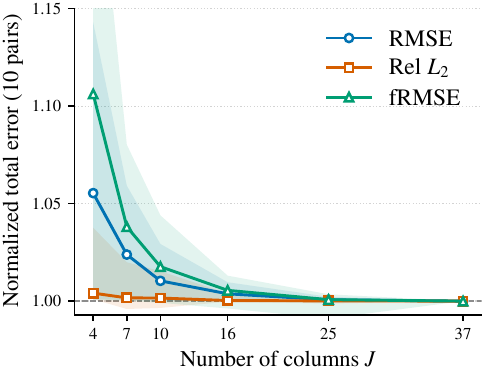}
\caption{\textbf{Number of columns.} Each curve is normalized by its value at thirty-seven columns, and every configuration consumes thirty-six evaluations of $\Phi$.}
\label{fig:readout-columns}
\vspace{-10pt}
\end{wrapfigure}

\paragraph{Two axes that consume no repair budget.} The retained depths and the ridge parameter enter only the closed-form solve of Eq.~\ref{eq:ridge-solution}, so neither adds an evaluation of $\Phi$, and every configuration in Figures~\ref{fig:readout-columns} and~\ref{fig:ridge} consumes the same thirty-six evaluations. Figure~\ref{fig:readout-columns} varies the number of columns $J$ by retaining fewer depths in $\mathcal{S}$, and although the rollout passes through every intermediate iterate so that four and thirty-seven columns cost the same, the error falls throughout. Summed over the ten pairs formed by FNO and DPOT-S, the RMSE drops by $5.2$ percent and the fRMSE by $9.6$ percent, whereas the relative $L_{2}$ error moves by $0.4$ percent, so the same axis is worth more than twenty times as much on one metric as on another. All three curves are flat beyond twenty-five columns, and we retain every depth at the operating point because nothing remains to select rather than because the last twelve columns improve accuracy.

\begin{wrapfigure}{r}{0.415\textwidth}
\vspace{-15pt}
\centering
\includegraphics[width=\linewidth]{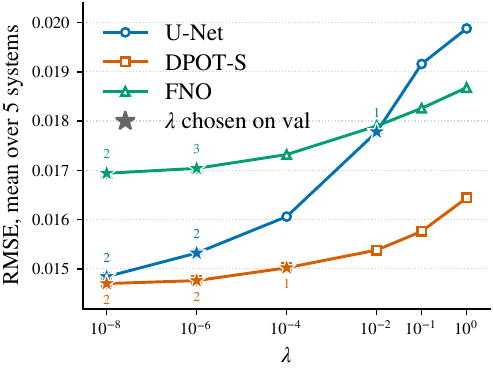}
\caption{\textbf{Ridge parameter.} RMSE averaged over the five systems. Dotted lines mark the source prediction, and stars mark the values selected on the fitting split, with the number of systems selecting each.}
\label{fig:ridge}
\vspace{-10pt}
\end{wrapfigure}

Figure~\ref{fig:ridge} varies the ridge parameter over eight orders of magnitude. Its two extremes bracket the method, since $\lambda\to0$ recovers the unregularized combination of Proposition~\ref{prop:composition} and $\lambda\to\infty$ drives every weight to zero and therefore returns the source prediction, because Eq.~\ref{eq:spectral-output} expresses the output as a correction to that prediction. The error accordingly rises by between ten and thirty-four percent across the sweep and then stops, and every curve stays below the corresponding source prediction throughout, so a poorly chosen $\lambda$ degrades the combination towards its starting point rather than away from it. The value selected on a held-out half of the fitting split lies on the flat left portion of every curve and coincides with the best value on the test split, or is indistinguishable from it at the reported precision. That value is selected per system rather than per backbone, so one backbone may receive different values in different systems.

% Complete the pending entries before drawing conclusions about IRNO or M=3.
% Existing M=1 values are retained from the preceding manuscript.
% Compute M=3 metrics from averaged predictions of two selected runs.
% An M=3 result or an average of individual-run errors is not an M=3 result.
% If repeated runs are summarized, identify the runs or ensemble pairs used.

\section{Conclusion}
\label{sec:conclusion}

We introduced $\text{R}^{2}\text{NO}$, which adapts a simulation-pretrained neural operator to real measurements by repairing its finetuned prediction with a shared module and composing the final field, cell by cell in the Fourier domain, from the source prediction and every iterate of the repair trajectory. On five RealPDEBench systems and six backbones, the composition consistently improves on both full finetuning and iterative refinement at the same repair budget, and the ablations attribute this gain to the cell-wise combination rather than to the additional repair modules. These results suggest that adaptation depth is better fitted per spectral component than chosen once for the entire field.
\paragraph{Limitations.} The fixed-point analysis is local and conditional, and the dominance guarantee holds on the fitting split rather than unseen data, so the cell-wise weights depend on a fitting split of adequate size. Spectral composition minimizes error rather than enforcing physical constraints, so the composed field need not satisfy the relations the governing equations impose on its Fourier coefficients. The three-module configuration triples the training and repair computation of iterative refinement, although the single-module configuration already improves on it at the same budget.
% \clearpage

\subsection*{AI Use Statement}
Generative AI tools were used to assist with revising the manuscript. Specifically, they were used to improve the clarity and presentation of the text during the editing process. The authors reviewed and verified all AI-assisted content and take full responsibility for the final manuscript.

\subsection*{Reproducibility Statement}

All experiments use the public RealPDEBench benchmark under its standard protocol. Section~\ref{sec:method} describes the method, Appendix~\ref{app:training-details} lists all training, fitting, and hardware settings, and Appendix~\ref{app:diagnostic-setup} describes the diagnostic of Figure~\ref{fig:dist_fourier}. We will release the code and trained repair modules upon publication.

\bibliographystyle{iclr2027_conference}
\bibliography{iclr2027_conference}

\clearpage
\appendix

\section{Related Work}
\label{sec:related}
\paragraph{Neural operators from simulation to measurements.}
Neural operators such as FNO~\citep{li2020fourier}, DeepONet~\citep{lu2021learning}, CNO~\citep{raonic2023convolutional}, and Transolver~\citep{wu2024transolver} learn solution maps of PDEs from simulated fields and now serve as surrogates in weather forecasting~\citep{pathak2022fourcastnet, bonev2025fourcastnet, bonev2023spherical}, satellite methane detection~\citep{heo2026flame}, subsurface carbon storage~\citep{wen2023real, chandra2025fourier}, and materials mechanics~\citep{you2022learning}. Pretraining on many simulated systems, as in DPOT~\citep{hao2024dpot}, yields representations that transfer across systems, but still within simulation. RealPDEBench~\citep{hu2026realpdebench} pairs such simulations with real measurements of the same systems, and its real-world finetuning task is the setting we adopt. Finetuning on a few real trajectories is the natural first step, but it can overwrite structure that pretraining had captured. 
A two-phase remedy, pretraining across a family of PDEs and then adapting only a handful of coefficients to a target instance, has been explored for physics-informed neural networks~\citep{cho2023hypernetwork, cho2024parameterized}. $\text{R}^{2}\text{NO}$ follows the two-phase algorithm in the output space. It finetunes first, then corrects the prediction of the finetuned operator with a repair module and decides, cell by cell in the Fourier domain, how much of that prediction to keep.
 
\paragraph{Iterative refinement and the combination of its iterates.}
Reusing one block over many steps separates computational depth from parameter count. Looped transformers unroll a weight-shared block for a fixed or budgeted number of iterations~\citep{giannou2023looped, saunshi2025reasoning}, LoopFormer trains on trajectories of varying length so that the budget can be chosen at test time~\citep{jeddi2026loopformer}, and deep equilibrium models run the iteration to its fixed point~\citep{bai2019deep, marwah2023deep}. IRNO brings this structure to neural operators as deterministic residual corrections of a frozen prediction, analyzed as a fixed-point iteration~\citep{liu2026iterative}. Adaptive-computation methods choose the number of steps per input or per token~\citep{graves2016adaptive}, but every one of these methods reads out a single depth for the unit it adapts, whereas $\text{R}^2\text{NO}$ keeps every iterate as a candidate. Combining candidates with weights fitted on held-out data is stacking~\citep{wolpert1992stacked, breiman1996bagging}, and averaging checkpoints along a training trajectory is a snapshot ensemble~\citep{huang2017snapshot}. $\text{R}^2\text{NO}$ stacks the finetuned prediction and the iterates of several repair modules with one set of weights per Fourier cell rather than one for the whole field.

\paragraph{The sim-to-real gap.}
Closing the gap between simulation and the real world matters across engineering, because simulators idealize the systems they model while measurements carry noise, unmodeled effects, and conditions the simulation never saw. Robotics has addressed this gap from the simulation side, by randomizing simulator parameters so that a policy becomes robust to the mismatch~\citep{tobin2017domain, peng2018sim} or by calibrating the simulator against real rollouts~\citep{chebotar2019closing}, and from the model side, by learning a residual that corrects the simulator's predictions on real data~\citep{ajay2018augmenting, zeng2020tossingbot}. The latter idea is older in statistics, where a calibrated computer model is completed by an additive discrepancy term fitted to observations~\citep{kennedy2001bayesian}. $\text{R}^{2}\text{NO}$ pursues the same goal for scientific machine learning, where neural surrogates of physical systems are trained on simulation and deployed on measurements. Its repair module is a discrepancy model applied to the prediction of a finetuned operator, and its spectral ensemble decides where the discrepancy needs correcting and where the source prediction already transfers.

\section{Analysis of Repair and Component Selection}
\label{app:analysis}

This appendix collects the analysis behind Sections~\ref{sec:motivation} and~\ref{sec:method}: convergence of the repair loop versus accuracy of its iterates, dominance of the fitted ensemble on the fitting split, and generalization of a cell-wise selection. Throughout, $\|\cdot\|_{2}$ is the Euclidean norm of a vectorized field, $P_{b}$ are the projections of Eq.~\ref{eq:projector}, and the member index of Eq.~\ref{eq:repair} is dropped.

\subsection{Convergence of the repair loop and accuracy of its iterates}
\label{sec:repair-theory}

Fix an input history $\mathcal{X}$ with target $y$ and write the loop of Eq.~\ref{eq:repair} as $h_{\ell+1}=T(h_{\ell})$ with $T(h):=h+\alpha\,\Phi(\mathcal{X},h;\theta)$. A fixed point $h^{\star}=T(h^{\star})$ is a prediction at which the learned correction vanishes. Let $d:=\|T(y)-y\|_{2}=\alpha\|\Phi(\mathcal{X},y;\theta)\|_{2}$ be the correction the module would still apply at the target, the quantity that the fixed-point penalty of Eq.~\ref{eq:training} penalizes, and assume that $T$ is a contraction with factor $q<1$ on a ball $\overline{B}_{r}(y)$ of radius $r$ around the target that contains $h_{0}$, with $d\leq(1-q)r$. Neither pretraining nor the fixed-point penalty proves these assumptions; the viewpoint is classical and is also used for learned refinement by \citet{bai2019deep,marwah2023deep,liu2026iterative}.

\begin{proposition}[Local convergence and target bias]
\label{prop:repair-fixed-point}
Under these assumptions the iterates converge to the unique fixed point $h^{\star}$ of $T$ in $\overline{B}_{r}(y)$, and with $e_{\ell}:=\|h_{\ell}-y\|_{2}$,
\begin{equation}
e_{\ell}\leq q^{\ell}e_{0}+\frac{1-q^{\ell}}{1-q}\,d
\quad\text{for every }\ell\geq0,
\qquad
\frac{d}{1+q}\leq\|h^{\star}-y\|_{2}\leq\frac{d}{1-q}.
\label{eq:repair-convergence}
\end{equation}
\end{proposition}

\begin{proof}
For $h\in\overline{B}_{r}(y)$, $\|T(h)-y\|_{2}\leq q\|h-y\|_{2}+d\leq r$, so $T$ maps the ball into itself and Banach's theorem gives the fixed point. The same inequality gives $e_{\ell+1}\leq qe_{\ell}+d$, which iterates to the first bound and, as $\ell\to\infty$, to the upper bias bound. The lower bound follows from $d\leq\|T(y)-T(h^{\star})\|_{2}+\|h^{\star}-y\|_{2}\leq(1+q)\|h^{\star}-y\|_{2}$.
\end{proof}

The first term contracts the error of the source prediction and the second is the price of an imperfect correction at the target. When $d>0$, an extra step is guaranteed to help only while $e_{\ell}>d/(1-q)$, and if $h_{0}=y$ the first repair moves away from the target although the loop is contractive, which is why the source prediction enters the ensemble as its own column. Convergence also says nothing about where each region attains its best error. Take $y=(0,0)$, $h_{0}=(-1,-3)$, and $T(h)=\tfrac{1}{2}h+\tfrac{1}{2}(1,1)$, a $\tfrac{1}{2}$-contraction with $h^{\star}=(1,1)$, $h_{1}=(0,-1)$, and $h_{2}=(\tfrac{1}{2},0)$. With $P_{1}$ and $P_{2}$ the projections onto the two coordinates,
\begin{equation}
\|P_{1}h_{1}+P_{2}h_{2}-y\|_{2}^{2}=0<\tfrac{1}{4}=\min_{\ell\geq0}\|h_{\ell}-y\|_{2}^{2},
\label{eq:repair-example}
\end{equation}
so the region-wise composition is exact while no single iterate is, and the composed field is not a fixed point of $T$. This is what Eq.~\ref{eq:depth-composition-gap} measures on real data.

\subsection{Dominance of the fitted combination on the fitting split}
\label{app:composition-proof}

For any weight array $\tilde{w}=(\tilde{w}_{b,j})$, let $\widehat{Q}(\tilde{w})$ be the fitting error of Eq.~\ref{eq:combination-objective} with $\tilde{w}$ in place of the fitted weights, and $\widehat{Q}_{b}(\tilde{w}_{b})$ its term for cell $b$, so that $\widehat{Q}(\tilde{w})=\sum_{b}\widehat{Q}_{b}(\tilde{w}_{b})$.

\begin{proposition}[Dominance on the fitting split]
\label{prop:composition}
Let $w^{\star}$ minimize $\widehat{Q}$; the weights of Eq.~\ref{eq:ridge-solution} with $\lambda=0$ do so whenever every $\mathbf{G}_{b}$ is invertible. Then $\widehat{Q}(w^{\star})\leq\widehat{Q}(\tilde{w})$ for every $\tilde{w}$, in particular for the source prediction ($\tilde{w}=0$), any single candidate ($\tilde{w}_{b}=e_{j}$ for all $b$, with $e_{j}$ the one-hot vector selecting column $j$), and any cell-wise selection ($\tilde{w}_{b}=e_{j(b)}$).
\end{proposition}

The proposition holds because each $\widehat{Q}_{b}$ depends only on $\tilde{w}_{b}$, so $\widehat{Q}$ is minimized by minimizing every cell separately, which is what the normal equations of Eq.~\ref{eq:normal-equations} do, while any other $\tilde{w}$ is merely one feasible choice in every cell.
For $M=1$ and $\mathcal{S}=\{1,\dots,L\}$, a cell-wise selection is the composite field $\sum_{b}P_{b}h_{\ell_{b}}$ of Section~\ref{sec:motivation}, so the common-depth gap of Eq.~\ref{eq:depth-composition-gap} is the one-hot special case of the proposition. Since the source prediction is that of the full finetuning baseline, the case $\tilde{w}=0$ shows that the ensemble cannot fit worse than that baseline. 
The argument covers the squared error of Eq.~\ref{eq:combination-objective}, but not the relative error or frequency-bin aggregates of the benchmark, and the ridge term gives up part of the inequality for stability.

\subsection{Generalization of a cell-wise selection}
\label{app:selection-generalization}

Proposition~\ref{prop:composition} concerns the fitting split. For $M=1$ and $\mathcal{S}=\{1,\dots,L\}$, let the fitting split consist of $n$ real trajectories, treated as the independent units because windows within a trajectory overlap, and let $\widehat{E}^{\mathrm{fit}}_{b}(\ell)$ be the regional error of Eq.~\ref{eq:diagnostic-region-error} on the fitting split, an average over trajectories of a per-trajectory error in $[0,C_{b}]$ whose expectation over the real distribution is $E_{b}(\ell)$. The repair module is trained on $\mathcal{D}^{\mathrm{train}}$, and one of $N_{\mathrm{ckpt}}$ saved checkpoints is selected on the fitting split ($N_{\mathrm{ckpt}}=12$ in our experiments), so the candidates of every checkpoint are independent of the fitting split. Hoeffding's inequality and a union bound over the $N_{\Omega}(L+1)N_{\mathrm{ckpt}}$ combinations of region, depth, and checkpoint give, with probability at least $1-\delta$ and simultaneously for every saved checkpoint, hence for the one selected on the fitting split,
\begin{equation}
\bigl|E_{b}(\ell)-\widehat{E}^{\mathrm{fit}}_{b}(\ell)\bigr|\leq\xi_{b}:=C_{b}\sqrt{\frac{\log\bigl(2N_{\Omega}(L+1)N_{\mathrm{ckpt}}/\delta\bigr)}{2n}}
\qquad\text{for all } b \text{ and } \ell\in\mathcal{S}_{0}.
\label{eq:uniform-selection}
\end{equation}
On this event a minimizer $\ell_{b}$ of $\widehat{E}^{\mathrm{fit}}_{b}$ satisfies $E_{b}(\ell_{b})\leq\min_{\ell}E_{b}(\ell)+2\xi_{b}$, and since Eq.~\ref{eq:parseval} makes the risk additive across regions, the composite field $\sum_{b}P_{b}h_{\ell_{b}}$ has population risk at most $\sum_{b}\min_{\ell\in\mathcal{S}_{0}}E_{b}(\ell)+2\sum_{b}\xi_{b}$. Region-wise selection thus pays $2\xi_{b}$ per region for choosing on data, a price that grows only logarithmically in the number of regions, depths, and checkpoints. Any other choice made on the fitting split, such as the ridge parameter, must be covered by the union bound or made on a separate split, and the fitted weights of Eq.~\ref{eq:ridge-solution}, which range over a continuous class, would need a uniform bound that we do not pursue.

\section{Experimental Setup}
\label{app:setup}
This appendix gives the details of the experiments in Section~\ref{sec:experiments}, including the setup of the diagnostic in Figure~\ref{fig:dist_fourier} and the data, models, training, fitting, and hardware settings behind Tables~\ref{tab:main} and~\ref{tab:ablation}.
\subsection{Frequency-resolved diagnostic of Figure~\ref{fig:dist_fourier}}
\label{app:diagnostic-setup}

\paragraph{Model and data.} The diagnostic uses the Cylinder system of RealPDEBench, with the finetuned FNO checkpoint of Section~\ref{sec:experiments} as the frozen source operator $\Psi(\cdot;\pi)$. The refinement module $\Phi(\cdot;\theta)$ is trained on $\mathcal{D}^{\mathrm{train}}$ with the objective of Eq.~\ref{eq:training} and the architecture, step size, and optimization settings of the IRNO baseline. After training, $\pi$ and $\theta$ are held fixed and Eq.~\ref{eq:irno-trajectory} is run for $L=12$ steps on every input history of the test split $\mathcal{D}^{\mathrm{test}}$, which yields the iterates $h_{0},\ldots,h_{L}$. No hyperparameter or design choice of Section~\ref{sec:method} is selected on this split.

\paragraph{Metric.} Let $\lvert\Omega_{b}\rvert$ denote the number of coefficients retained by the mask of region $b$ across measured channels and forecast frames. The fRMSE of region $b$ at depth $\ell$ and its percentage change relative to depth zero, which Figure~\ref{fig:dist_fourier}(c) reports, are
\begin{align}
\operatorname{fRMSE}_{b}(\ell)&=\Bigl(\widehat{\mathbb{E}}_{\mathrm{test}}\,\tfrac{1}{\lvert\Omega_{b}\rvert}\bigl\|\mathbf{1}_{\Omega_{b}}\odot\mathcal{F}(h_{\ell}-y)\bigr\|_{2}^{2}\Bigr)^{1/2},
\label{eq:diagnostic-frmse}\\
\Delta\operatorname{fRMSE}_{b}(\ell)&=100\,\frac{\operatorname{fRMSE}_{b}(\ell)-\operatorname{fRMSE}_{b}(0)}{\operatorname{fRMSE}_{b}(0)}.
\label{eq:diagnostic-frmse-change}
\end{align}
Negative values indicate improvement. Because $\mathcal{F}$ is unitary, the regional squared error of Eq.~\ref{eq:diagnostic-region-error} satisfies $\widehat{E}_{b}(\ell)=\lvert\Omega_{b}\rvert\operatorname{fRMSE}_{b}(\ell)^{2}$, so the markers in Figure~\ref{fig:dist_fourier}(c), which mark the minimum over $\ell\in\mathcal{S}_{0}$, are also the region-wise minimizers in Eq.~\ref{eq:depth-composition-gap}.

\subsection{Training and fitting details}
\label{app:training-details}

\paragraph{Data and splits.} We use the real-world finetuning task of RealPDEBench for all five systems. The training split of the benchmark serves as $\mathcal{D}^{\mathrm{train}}$ for the repair modules, its validation split serves in full as the fitting split $\mathcal{D}^{\mathrm{fit}}$ of the spectral ensemble, and its test split is used only to compute the reported errors. Each output window contains $20$ forecast frames, except for Controlled Cylinder with $10$, on a grid of $64\times128$ points for Cylinder, Controlled Cylinder, and Foil and $64\times64$ points for FSI and Combustion. The measured channels are read from the benchmark data and are two of the three reported channels for the fluid systems and one of the sixteen reported channels for Combustion.

\paragraph{Source operators.} For U-Net, DeepONet, FNO, Transolver, and DPOT-S, the source operator is the finetuned checkpoint released with RealPDEBench. For CNO we finetune the simulation-pretrained backbone ourselves under the benchmark protocol, because the released checkpoint does not work. The source operator is frozen after finetuning, and its prediction $h_{0}$ is computed once per window and shared by every repair module.

\paragraph{Repair modules.} The repair module $\Phi$ is a U-Net with base width $32$ that receives the input history and the current iterate, concatenated over all forecast frames along the channel axis, and returns an update with the shape of the output field, which amounts to about $8.7$ million parameters for Cylinder. It is trained for $12$ epochs with learning rate $3\times10^{-4}$ and step size $\alpha=0.2$, supervising $12$ iterates of the loop with the objective of Eq.~\ref{eq:training}. The spatial term is the mean squared error over the measured channels, the spectral term compares the magnitudes of the real-input Fourier coefficients normalized by the number of grid points with the exponent $\eta_{\ell}$ increasing linearly from one to two, and the fixed-point term is evaluated once per batch at the target, with weights $\beta_{\mathrm{spe}}=1.0$ and $\beta_{\mathrm{fp}}=0.01$. The checkpoint is selected among the twelve saved epochs by the error on the validation split, which is also the fitting split of the ensemble, and Appendix~\ref{app:selection-generalization} accounts for this dependence.

\paragraph{Iterative refinement baseline.} IRNO uses the repair module trained with a fixed seed of $42$ and is reported at depth $L=12$ and, in Table~\ref{tab:main}, at the depth $\ell^{\star}$ selected on the fitting split, so it shares the trained trajectory with $\text{R}^{2}\text{NO}$ with $M=1$ and differs from it only in the readout. The loss terms, their weights, the step size, and the learning rate follow the published configuration of the baseline. Its number of epochs, base width, and padding differ from that configuration because the benchmark differs, and these differences, together with the normalization of the Fourier transform in the spectral term and the absence of a warmup on the fixed-point term, apply identically to IRNO and to $\text{R}^{2}\text{NO}$.

\paragraph{Spectral ensemble.} The Fourier transform is applied over the two spatial axes of the denormalized output field, and the cells of Eq.~\ref{eq:cell-index} use $N_{\rho}=128$ radial bands and $N_{\varphi}=16$ angular sectors, shared across forecast frames. Because the partition is finer than the Fourier grid near the origin, $1{,}468$ of the $2{,}048$ requested spatial cells are occupied on the $64\times128$ grid and $1{,}182$ on the $64\times64$ grid, and empty cells receive zero weight. With all twelve depths retained, the ensemble has $J=13$ columns for $M=1$ and $J=37$ for $M=3$, which gives $151{,}552$ weights for Cylinder with $M=3$. The Gram matrices and right-hand sides of Eq.~\ref{eq:normal-equations} are accumulated in double precision over every window of the fitting split, each window entering with weight $\|y\|_{2}^{-1}$, a fixed reweighting of Eq.~\ref{eq:combination-objective} that lies between the unweighted squared error and the relative error and leaves Proposition~\ref{prop:composition} unchanged. The ridge parameter is selected from $\{10^{-8},10^{-6},10^{-4},10^{-2},10^{-1},1\}$ by solving on one half of the fitting split and scoring on the other half, after which the system is solved again on the full split. Each cell is fitted from many more observations than it has weights, since for Cylinder the fitting split contains $4{,}820$ windows and even a cell with a single Fourier coefficient accumulates $96{,}400$ observations for its $37$ weights.

\paragraph{Computational cost.} Training a repair module and fitting the ensemble are the only adaptation steps, and the latter is a closed-form solve that requires no gradient updates. At inference, $\text{R}^{2}\text{NO}$ with $M=1$ evaluates $\Phi$ twelve times, exactly as IRNO does, and $\text{R}^{2}\text{NO}$ with $M=3$ evaluates it thirty-six times. The fitted weights are expanded once into a tensor over Fourier coefficients, so the ensemble adds only one forward and one inverse transform per window to the cost of the rollout. Experiments run on a server with four NVIDIA RTX A6000 GPUs of 48 GB each, two AMD EPYC 9224 processors with 48 cores in total, and 377 GB of memory, using PyTorch 2.7.0 with CUDA 12.6 in single precision, while the Gram matrices of the spectral ensemble are accumulated in double precision. We use version 0.1.0 of the RealPDEBench code at commit \texttt{62f4c80} and version 2.0.0 of its data.

\section{Additional results}
\label{app:additional-tables}

Table~\ref{tab:rel-l2} reports the relative $L_{2}$ error for every setting of Table~\ref{tab:main}. $\text{R}^{2}\text{NO}$ with $M=3$ improves on full finetuning in twenty-nine of the thirty pairs and $\text{R}^{2}\text{NO}$ with $M=1$ in twenty-eight, compared with twenty-two for IRNO. At the matched budget of twelve evaluations of $\Phi$, $\text{R}^{2}\text{NO}$ with $M=1$ improves on IRNO in twenty-six pairs, and the four exceptions all lie on Cylinder, where IRNO attains the lowest relative $L_{2}$ error for U-Net, CNO, and DPOT-S and full finetuning does so for FNO. This metric normalizes each window by the norm of its target and therefore weights windows differently from the squared error that the combination minimizes on the fitting split. On Cylinder a small fraction of windows carries most of the squared error, so a combination fitted to that error need not be optimal under the relative normalization, while on the remaining four systems $\text{R}^{2}\text{NO}$ with $M=3$ attains the lowest relative $L_{2}$ error for every backbone.

\begin{table}[h]
\centering
\caption{Relative $L_{2}$ error ($\times10^{-2}$, lower is better) on RealPDEBench for the settings of Table~\ref{tab:main}. The repair budget counts evaluations of $\Phi$, and the best value of each column within a backbone is set in bold.}
\label{tab:rel-l2}
\small
\setlength{\tabcolsep}{4pt}
\renewcommand{\arraystretch}{1.04}
\begin{tabular}{@{}llc*{5}{c}@{}}
\toprule
Backbone & Adaptation & \shortstack{Repair\\budget} & Cylinder & \shortstack{Controlled\\Cylinder} & FSI & Foil & Combustion \\
\midrule
\multirow{4}{*}{U-Net}
    & Full finetuning & 0 & 7.284 & 5.435 & 5.789 & 1.447 & 54.025 \\
    & IRNO & 12 & \textbf{6.926} & 5.421 & 6.053 & 1.454 & 55.257 \\
    & $\text{R}^{2}\text{NO}$ ($M{=}1$) & 12 & 7.652 & 5.317 & 5.753 & 1.414 & 53.840 \\
    & $\text{R}^{2}\text{NO}$ ($M{=}3$) & 36 & 7.103 & \textbf{5.314} & \textbf{5.751} & \textbf{1.412} & \textbf{53.789} \\
\midrule
\multirow{4}{*}{DeepONet}
    & Full finetuning & 0 & 15.031 & 22.837 & 23.675 & 3.751 & 57.215 \\
    & IRNO & 12 & 8.468 & 6.678 & 7.640 & 2.869 & 55.707 \\
    & $\text{R}^{2}\text{NO}$ ($M{=}1$) & 12 & 8.395 & 6.041 & 6.772 & 2.436 & 55.147 \\
    & $\text{R}^{2}\text{NO}$ ($M{=}3$) & 36 & \textbf{7.745} & \textbf{5.848} & \textbf{6.429} & \textbf{2.122} & \textbf{54.267} \\
\midrule
\multirow{4}{*}{CNO}
    & Full finetuning & 0 & 8.298 & 5.672 & 6.788 & 2.055 & 58.435 \\
    & IRNO & 12 & \textbf{7.060} & 5.517 & 6.054 & 2.115 & 56.517 \\
    & $\text{R}^{2}\text{NO}$ ($M{=}1$) & 12 & 7.384 & 5.490 & 6.032 & 1.719 & 55.716 \\
    & $\text{R}^{2}\text{NO}$ ($M{=}3$) & 36 & 7.063 & \textbf{5.482} & \textbf{5.972} & \textbf{1.660} & \textbf{54.963} \\
\midrule
\multirow{4}{*}{FNO}
    & Full finetuning & 0 & \textbf{7.804} & 7.025 & 8.807 & 2.056 & 56.793 \\
    & IRNO & 12 & 7.940 & 5.784 & 6.927 & 2.159 & 57.258 \\
    & $\text{R}^{2}\text{NO}$ ($M{=}1$) & 12 & 7.994 & 5.745 & 6.786 & 1.767 & 55.579 \\
    & $\text{R}^{2}\text{NO}$ ($M{=}3$) & 36 & 7.839 & \textbf{5.673} & \textbf{6.584} & \textbf{1.734} & \textbf{55.387} \\
\midrule
\multirow{4}{*}{Transolver}
    & Full finetuning & 0 & 17.063 & 11.653 & 14.900 & 3.352 & 78.024 \\
    & IRNO & 12 & 8.791 & 6.221 & 7.120 & 2.101 & 64.365 \\
    & $\text{R}^{2}\text{NO}$ ($M{=}1$) & 12 & 7.992 & 5.959 & 6.758 & 1.867 & 56.295 \\
    & $\text{R}^{2}\text{NO}$ ($M{=}3$) & 36 & \textbf{7.403} & \textbf{5.769} & \textbf{6.650} & \textbf{1.839} & \textbf{55.381} \\
\midrule
\multirow{4}{*}{DPOT-S}
    & Full finetuning & 0 & 7.660 & 6.151 & 7.007 & 1.733 & 53.786 \\
    & IRNO & 12 & \textbf{6.919} & 5.599 & 6.043 & 1.684 & 55.378 \\
    & $\text{R}^{2}\text{NO}$ ($M{=}1$) & 12 & 7.379 & 5.533 & 6.019 & 1.634 & 53.152 \\
    & $\text{R}^{2}\text{NO}$ ($M{=}3$) & 36 & 6.941 & \textbf{5.470} & \textbf{5.984} & \textbf{1.623} & \textbf{53.028} \\
\bottomrule
\end{tabular}
\end{table}

\section{Additional ablations}
\label{app:additional-ablations}

\paragraph{Axes of the partition.} Table~\ref{tab:ablation-axes} removes one axis of the partition at a time while keeping the fitted combination, the three repair modules, and the repair budget of the operating point. Summed over the fifteen pairs, the RMSE rises from $0.2323$ for the full partition to $0.2408$ without the angular axis, $0.2478$ without the radial axis, and $0.2446$ without the channel axis, against $0.2812$ for a single cell, so each axis recovers part of the gap and none of them accounts for it alone. The fRMSE shows the same ordering more sharply, rising from $0.02794$ to $0.02917$, $0.03150$, and $0.03105$ respectively, against $0.03682$ for a single cell, so the radial axis matters most for the frequency-domain error, consistent with the dependence of the best depth on frequency magnitude in Section~\ref{sec:motivation}. The effect of the individual axes is concentrated on Cylinder and lies at or below the reported precision on the remaining systems. Removing the channel axis leaves Combustion unchanged, because that system has a single measured channel and the corresponding rows are identical by construction.

\paragraph{Base column.} Table~\ref{tab:ablation-base} removes the base column $-h_{0}$ of Eq.~\ref{eq:columns}. Without it, the prediction in each cell is an affine combination of the source prediction and the retained iterates, and it still reproduces the source prediction when every weight vanishes, because Eq.~\ref{eq:spectral-output} expresses the output as a correction to that prediction. The base column lifts the constraint that the coefficients sum to one and thereby lets a cell shrink its overall amplitude, which reduces the squared error in cells where every candidate carries more energy than the target. Removing it raises the summed RMSE over the fifteen pairs from $0.2323$ to $0.2378$, and at full precision the RMSE increases in every pair, with the largest change of $10.9$ percent for DPOT-S on Cylinder. The summed fRMSE rises from $0.02794$ to $0.02818$, although the column slightly increases the fRMSE in five pairs, most visibly by $1.8$ percent for DPOT-S on Foil, since it is fitted to the squared error on the fitting split and need not reduce the frequency-domain error in every pair. The combination without the column still improves on full finetuning in every pair, which confirms that retention of the source prediction follows from the residual form of the output rather than from the column.

\begin{table}[h]
\centering
\caption{Removing one axis of the partition at a matched repair budget. All rows use three repair modules rolled out to $L=12$ and fitted weights, and differ only in the partition. The RMSE and the fRMSE are reported, lower being better, and the best value of each column within a backbone is set in bold, with ties at the reported precision marked jointly.}
\label{tab:ablation-axes}
\small
\setlength{\tabcolsep}{2.5pt}
\renewcommand{\arraystretch}{1.04}
\resizebox{\columnwidth}{!}{%
\begin{tabular}{@{}ll*{10}{c}@{}}
\toprule
& & \multicolumn{2}{c}{Cylinder}
    & \multicolumn{2}{c}{\shortstack{Controlled\\Cylinder}}
    & \multicolumn{2}{c}{FSI}
    & \multicolumn{2}{c}{Foil}
    & \multicolumn{2}{c}{Combustion} \\
\cmidrule(lr){3-4}\cmidrule(lr){5-6}\cmidrule(lr){7-8}
\cmidrule(lr){9-10}\cmidrule(lr){11-12}
Backbone & Partition
    & \shortstack{RMSE\\$(\times10^{-2})$} & \shortstack{fRMSE\\$(\times10^{-3})$}
    & \shortstack{RMSE\\$(\times10^{-2})$} & \shortstack{fRMSE\\$(\times10^{-3})$}
    & \shortstack{RMSE\\$(\times10^{-2})$} & \shortstack{fRMSE\\$(\times10^{-3})$}
    & \shortstack{RMSE\\$(\times10^{-2})$} & \shortstack{fRMSE\\$(\times10^{-3})$}
    & \shortstack{RMSE\\$(\times10^{-2})$} & \shortstack{fRMSE\\$(\times10^{-3})$} \\
\midrule
\multirow{5}{*}{U-Net}
    & Single cell & 5.213 & 8.110 & 0.778 & 0.912 & 0.837 & 0.699 & 0.921 & 0.818 & 2.131 & 2.497 \\
    & No radial axis & 3.606 & 5.709 & \textbf{0.775} & 0.905 & \textbf{0.836} & \textbf{0.697} & 0.908 & 0.794 & 2.121 & 2.468 \\
    & No angular axis & 3.154 & 4.473 & 0.776 & \textbf{0.904} & \textbf{0.836} & \textbf{0.697} & 0.910 & 0.793 & 2.122 & 2.468 \\
    & No channel axis & 3.412 & 5.557 & 0.776 & 0.908 & \textbf{0.836} & \textbf{0.697} & 0.910 & 0.798 & \textbf{2.119} & \textbf{2.462} \\
    & Full partition & \textbf{2.775} & \textbf{3.912} & 0.776 & 0.905 & \textbf{0.836} & 0.698 & \textbf{0.905} & \textbf{0.784} & \textbf{2.119} & \textbf{2.462} \\
\midrule
\multirow{5}{*}{FNO}
    & Single cell & 4.463 & 6.925 & 0.808 & 0.946 & 1.010 & 0.971 & 1.120 & 1.100 & 2.235 & 2.715 \\
    & No radial axis & 3.810 & 5.743 & 0.805 & 0.937 & 0.986 & 0.928 & 1.072 & 1.040 & 2.204 & 2.665 \\
    & No angular axis & 3.682 & 5.260 & 0.805 & \textbf{0.934} & 0.989 & 0.929 & 1.081 & 1.045 & 2.206 & 2.663 \\
    & No channel axis & 3.707 & 5.537 & 0.805 & 0.937 & 0.991 & 0.939 & 1.076 & 1.048 & \textbf{2.201} & \textbf{2.659} \\
    & Full partition & \textbf{3.432} & \textbf{4.919} & \textbf{0.804} & \textbf{0.934} & \textbf{0.985} & \textbf{0.927} & \textbf{1.061} & \textbf{1.019} & \textbf{2.201} & \textbf{2.659} \\
\midrule
\multirow{5}{*}{DPOT-S}
    & Single cell & 3.829 & 5.990 & 0.793 & 0.931 & 0.864 & 0.729 & 1.048 & 1.139 & 2.071 & 2.333 \\
    & No radial axis & 2.931 & 4.536 & \textbf{0.789} & 0.917 & \textbf{0.860} & \textbf{0.726} & 1.022 & 1.111 & 2.059 & 2.320 \\
    & No angular axis & 2.792 & 3.933 & \textbf{0.789} & \textbf{0.912} & 0.861 & \textbf{0.726} & 1.020 & 1.110 & 2.060 & 2.319 \\
    & No channel axis & 2.891 & 4.413 & \textbf{0.789} & 0.913 & 0.861 & \textbf{0.726} & 1.028 & 1.133 & \textbf{2.058} & \textbf{2.318} \\
    & Full partition & \textbf{2.617} & \textbf{3.665} & \textbf{0.789} & 0.914 & \textbf{0.860} & \textbf{0.726} & \textbf{1.016} & \textbf{1.102} & \textbf{2.058} & \textbf{2.318} \\
\bottomrule
\end{tabular}}
\end{table}

\begin{table}[h]
\centering
\caption{Effect of the base column at the operating point. Both rows use three repair modules rolled out to $L=12$, fitted weights, and the full partition. The RMSE and the fRMSE are reported, lower being better, and the better value of each pair is set in bold, with ties at the reported precision marked jointly.}\label{tab:ablation-base}
\small
\setlength{\tabcolsep}{2.5pt}
\renewcommand{\arraystretch}{1.04}
\resizebox{\columnwidth}{!}{%
\begin{tabular}{@{}ll*{10}{c}@{}}
\toprule
& & \multicolumn{2}{c}{Cylinder}
    & \multicolumn{2}{c}{\shortstack{Controlled\\Cylinder}}
    & \multicolumn{2}{c}{FSI}
    & \multicolumn{2}{c}{Foil}
    & \multicolumn{2}{c}{Combustion} \\
\cmidrule(lr){3-4}\cmidrule(lr){5-6}\cmidrule(lr){7-8}
\cmidrule(lr){9-10}\cmidrule(lr){11-12}
Backbone & Base column
    & \shortstack{RMSE\\$(\times10^{-2})$} & \shortstack{fRMSE\\$(\times10^{-3})$}
    & \shortstack{RMSE\\$(\times10^{-2})$} & \shortstack{fRMSE\\$(\times10^{-3})$}
    & \shortstack{RMSE\\$(\times10^{-2})$} & \shortstack{fRMSE\\$(\times10^{-3})$}
    & \shortstack{RMSE\\$(\times10^{-2})$} & \shortstack{fRMSE\\$(\times10^{-3})$}
    & \shortstack{RMSE\\$(\times10^{-2})$} & \shortstack{fRMSE\\$(\times10^{-3})$} \\
\midrule
\multirow{2}{*}{U-Net}
    & Without & 2.805 & 3.953 & \textbf{0.776} & \textbf{0.904} & \textbf{0.836} & \textbf{0.697} & 0.911 & 0.788 & 2.128 & 2.480 \\
    & With & \textbf{2.775} & \textbf{3.912} & \textbf{0.776} & 0.905 & \textbf{0.836} & 0.698 & \textbf{0.905} & \textbf{0.784} & \textbf{2.119} & \textbf{2.462} \\
\midrule
\multirow{2}{*}{FNO}
    & Without & 3.534 & 4.984 & 0.806 & \textbf{0.934} & 0.988 & \textbf{0.927} & 1.091 & 1.021 & 2.222 & 2.678 \\
    & With & \textbf{3.432} & \textbf{4.919} & \textbf{0.804} & \textbf{0.934} & \textbf{0.985} & \textbf{0.927} & \textbf{1.061} & \textbf{1.019} & \textbf{2.201} & \textbf{2.659} \\
\midrule
\multirow{2}{*}{DPOT-S}
    & Without & 2.938 & 3.765 & 0.790 & \textbf{0.914} & 0.862 & \textbf{0.726} & 1.024 & \textbf{1.083} & 2.068 & 2.323 \\
    & With & \textbf{2.617} & \textbf{3.665} & \textbf{0.789} & \textbf{0.914} & \textbf{0.860} & \textbf{0.726} & \textbf{1.016} & 1.102 & \textbf{2.058} & \textbf{2.318} \\
\bottomrule
\end{tabular}}
\end{table}

\clearpage
\section{Adaptation from the Simulation-Pretrained Checkpoint}
\label{app:pretrained-source}

We also evaluate the setting in which the backbone is never finetuned: the simulation-pretrained operator is frozen and provides the source prediction $h_{0}$, and IRNO and $\text{R}^{2}\text{NO}$ are trained and fitted on top of it with the settings of Appendix~\ref{app:training-details}. Table~\ref{tab:pretrained-source} reports the results for FNO and DPOT-S, with full finetuning shown for reference only; its prediction is not a candidate of the ensemble.

\begin{table}[ht!]
\centering
\caption{Adaptation from the frozen simulation-pretrained checkpoint. A dagger marks a value lower than full finetuning of the same backbone. Lower is better, and the best value of each column within a backbone is set in bold.}
\label{tab:pretrained-source}
\small
\setlength{\tabcolsep}{2.5pt}
\renewcommand{\arraystretch}{1.04}
\resizebox{\columnwidth}{!}{%
\begin{tabular}{@{}ll*{10}{c}@{}}
\toprule
& & \multicolumn{2}{c}{Cylinder}
    & \multicolumn{2}{c}{\shortstack{Controlled\\Cylinder}}
    & \multicolumn{2}{c}{FSI}
    & \multicolumn{2}{c}{Foil}
    & \multicolumn{2}{c}{Combustion} \\
\cmidrule(lr){3-4}\cmidrule(lr){5-6}\cmidrule(lr){7-8}
\cmidrule(lr){9-10}\cmidrule(lr){11-12}
Backbone & Adaptation
    & \shortstack{RMSE\\$(\times10^{-2})$} & \shortstack{fRMSE\\$(\times10^{-3})$}
    & \shortstack{RMSE\\$(\times10^{-2})$} & \shortstack{fRMSE\\$(\times10^{-3})$}
    & \shortstack{RMSE\\$(\times10^{-2})$} & \shortstack{fRMSE\\$(\times10^{-3})$}
    & \shortstack{RMSE\\$(\times10^{-2})$} & \shortstack{fRMSE\\$(\times10^{-3})$}
    & \shortstack{RMSE\\$(\times10^{-2})$} & \shortstack{fRMSE\\$(\times10^{-3})$} \\
\midrule
\multirow{5}{*}{FNO}
    & Pretrained (source) & 7.389 & 11.481 & 2.845 & 5.934 & 4.263 & 5.899 & 2.738 & 3.733 & 3.630 & 5.147 \\
    & Full finetuning & 5.449 & 8.661 & 0.943 & 1.077 & 1.266 & 1.162 & 1.199 & 1.241 & \textbf{2.246} & \textbf{2.726} \\
    & IRNO ($\ell{=}12$) & 3.194$^{\dagger}$ & 4.575$^{\dagger}$ & 0.890$^{\dagger}$ & 1.124 & 0.981$^{\dagger}$ & 0.927$^{\dagger}$ & 1.546 & 1.794 & 2.654 & 3.482 \\
    & $\text{R}^{2}\text{NO}$ ($M{=}1$) & 2.787$^{\dagger}$ & 3.545$^{\dagger}$ & 0.873$^{\dagger}$ & 1.053$^{\dagger}$ & 0.935$^{\dagger}$ & 0.818$^{\dagger}$ & 1.177$^{\dagger}$ & 1.309 & 2.521 & 3.301 \\
    & $\text{R}^{2}\text{NO}$ ($M{=}3$) & \textbf{2.456}$^{\dagger}$ & \textbf{2.981}$^{\dagger}$ & \textbf{0.843}$^{\dagger}$ & \textbf{0.997}$^{\dagger}$ & \textbf{0.912}$^{\dagger}$ & \textbf{0.773}$^{\dagger}$ & \textbf{1.114}$^{\dagger}$ & \textbf{1.151}$^{\dagger}$ & 2.347 & 2.932 \\
\midrule
\multirow{5}{*}{DPOT-S}
    & Pretrained (source) & 5.124 & 7.595 & 2.333 & 3.939 & 2.594 & 3.054 & 2.210 & 2.712 & 3.932 & 5.418 \\
    & Full finetuning & 4.404 & 6.726 & 0.847 & 0.987 & 0.990 & \textbf{0.737} & 1.093 & 1.178 & \textbf{2.106} & \textbf{2.408} \\
    & IRNO ($\ell{=}12$) & 3.266$^{\dagger}$ & 4.260$^{\dagger}$ & 0.909 & 1.229 & 1.087 & 1.063 & 1.361 & 1.571 & 2.664 & 3.530 \\
    & $\text{R}^{2}\text{NO}$ ($M{=}1$) & 2.796$^{\dagger}$ & 3.430$^{\dagger}$ & 0.839$^{\dagger}$ & 1.039 & 1.013 & 0.913 & 1.079$^{\dagger}$ & 1.125$^{\dagger}$ & 2.490 & 3.250 \\
    & $\text{R}^{2}\text{NO}$ ($M{=}3$) & \textbf{2.258}$^{\dagger}$ & \textbf{2.544}$^{\dagger}$ & \textbf{0.814}$^{\dagger}$ & \textbf{0.964}$^{\dagger}$ & \textbf{0.951}$^{\dagger}$ & 0.806 & \textbf{1.022}$^{\dagger}$ & \textbf{0.986}$^{\dagger}$ & 2.446 & 3.161 \\
\bottomrule
\end{tabular}}
\end{table}

All adapted methods improve on the source prediction in every pair, as Proposition~\ref{prop:composition} guarantees for $\text{R}^{2}\text{NO}$ on the fitting split. At the same repair budget, $\text{R}^{2}\text{NO}$ with $M=1$ improves on IRNO in every pair, lowering the summed RMSE by $11.0$ percent and the summed fRMSE by $16.0$ percent. It also surpasses full finetuning in seven pairs under RMSE, against four for IRNO, and $\text{R}^{2}\text{NO}$ with $M=3$ does so in eight, without updating any parameter of the pretrained operator. The main exception is Combustion, where full finetuning remains best.

\section{Qualitative results}
\label{app:qualitative}

Figures~\ref{fig:qual-cylinder}--\ref{fig:qual-combustion} compare the predictions of DPOT-S on one test window of each system at five forecast frames. The rows show the measured target, full finetuning, IRNO, and $\text{R}^{2}\text{NO}$ with $M=3$, and all rows of a figure share one colour scale. At this scale the adapted methods differ only slightly, in line with the errors of Table~\ref{tab:main}, and the difference is most visible in the background of Foil, where full finetuning and IRNO carry a grid-aligned high-frequency pattern that $\text{R}^{2}\text{NO}$ largely removes.

\begin{figure}[h]
\centering
\includegraphics[width=0.95\linewidth]{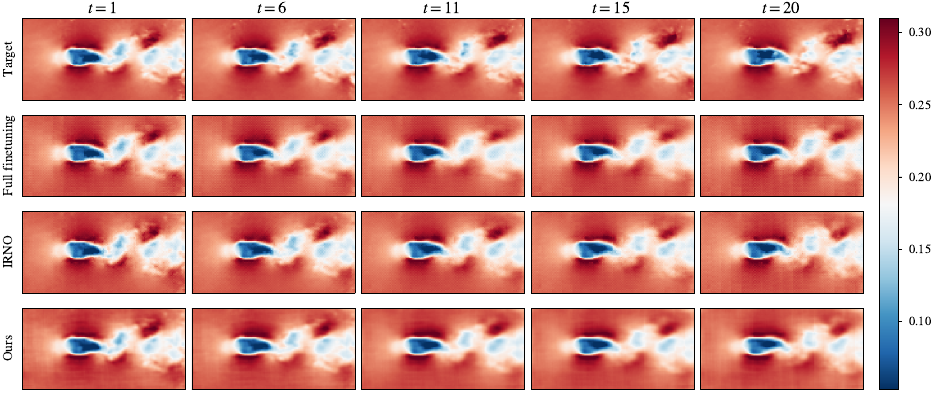}
\caption{Predictions of DPOT-S on one Cylinder test window at five forecast frames. Rows show the measured target, full finetuning, IRNO, and $\text{R}^{2}\text{NO}$; all panels share one color scale.}
\label{fig:qual-cylinder}
\end{figure}

\begin{figure}[h]
\centering
\includegraphics[width=0.95\linewidth]{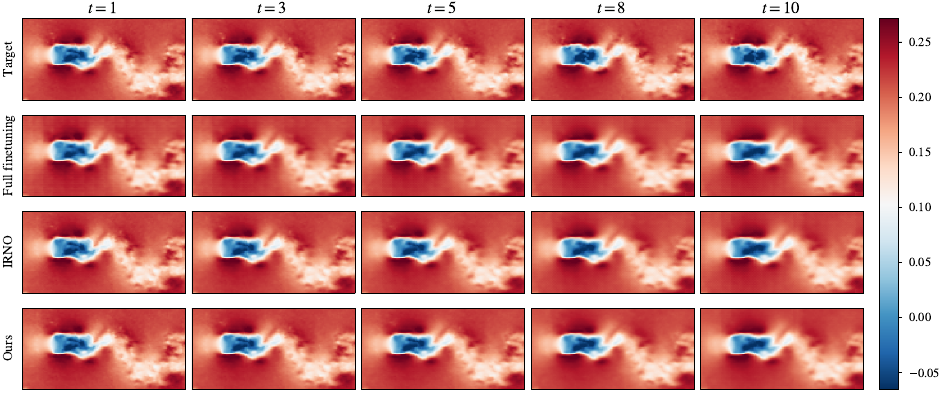}
\caption{Predictions of DPOT-S on one Controlled Cylinder test window at five forecast frames. Rows show the measured target, full finetuning, IRNO, and $\text{R}^{2}\text{NO}$; all panels share one color scale.}
\label{fig:qual-controlled-cylinder}
\end{figure}

\begin{figure}[h]
\centering
\includegraphics[width=0.9\linewidth]{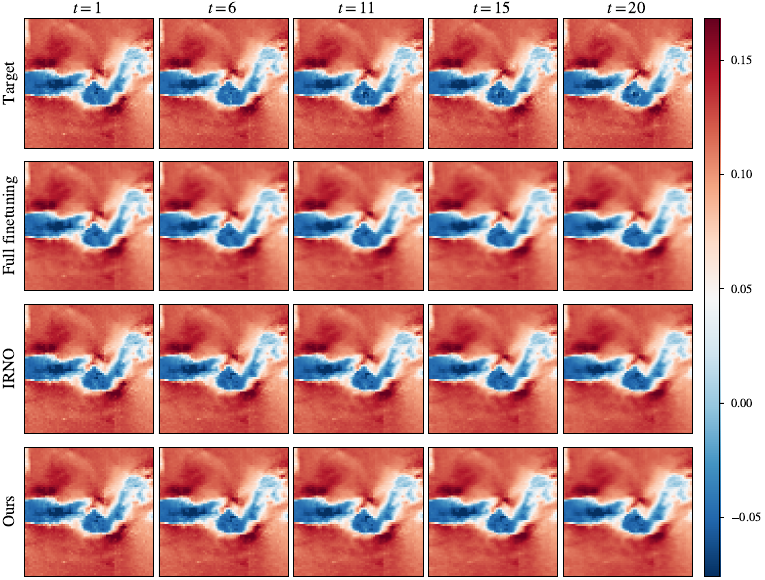}
\caption{Predictions of DPOT-S on one FSI test window at five forecast frames. Rows show the measured target, full finetuning, IRNO, and $\text{R}^{2}\text{NO}$; all panels share one color scale.}
\label{fig:qual-fsi}
\end{figure}

\begin{figure}[h]
\centering
\includegraphics[width=0.95\linewidth]{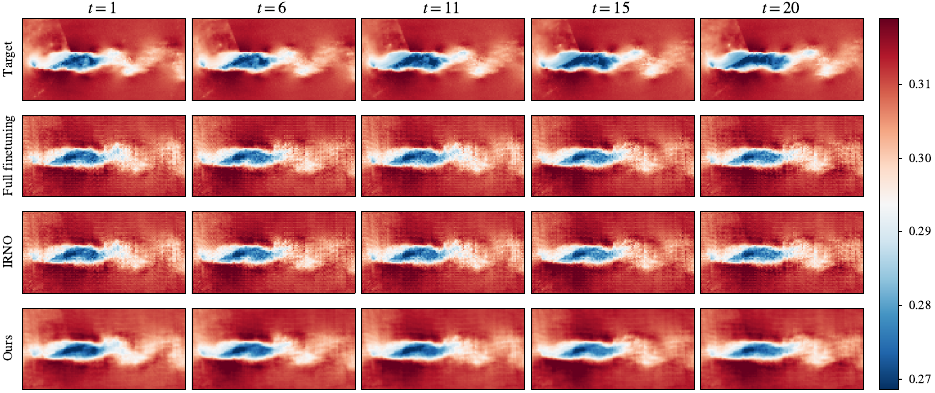}
\caption{Predictions of DPOT-S on one Foil test window at five forecast frames. Rows show the measured target, full finetuning, IRNO, and $\text{R}^{2}\text{NO}$; all panels share one color scale.}
\label{fig:qual-foil}
\end{figure}

\begin{figure}[h]
\centering
\includegraphics[width=0.9\linewidth]{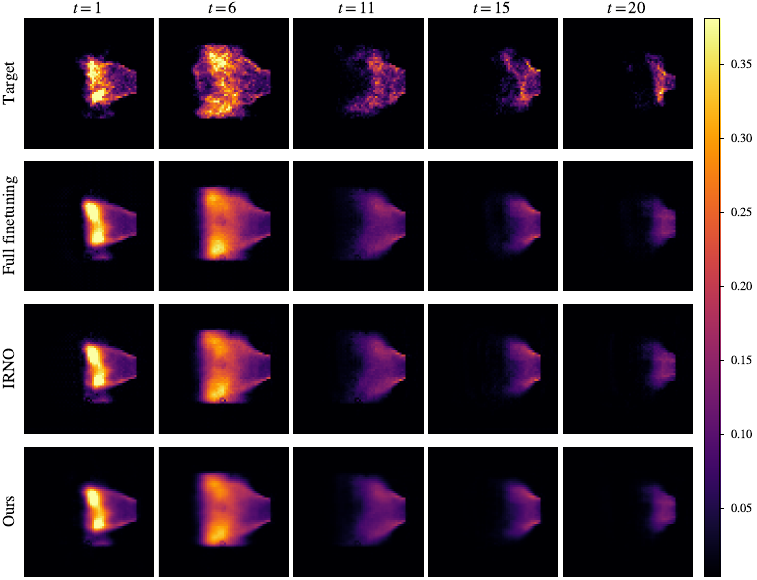}
\caption{Predictions of DPOT-S on one Combustion test window at five forecast frames. Rows show the measured target, full finetuning, IRNO, and $\text{R}^{2}\text{NO}$; all panels share one color scale.}
\label{fig:qual-combustion}
\end{figure}

\end{document}